\documentclass[11pt]{article}
\usepackage[square]{natbib}
\AtBeginDocument{}
\usepackage{tikz}

\usetikzlibrary{matrix,decorations.pathreplacing}

\usepackage{xcolor}

\usepackage{easybmat}
\usepackage{multirow,bigdelim}
\usepackage{amsmath}
\usepackage{amsthm}
\usepackage{amssymb}
\usepackage{algorithm}
\usepackage{subcaption}
\usepackage{color}
\usepackage[english]{babel}
\usepackage{graphicx}
\usepackage{wrapfig,epsfig}
\usepackage{epstopdf}
\usepackage{url}
\usepackage{graphicx}
\usepackage{color}
\usepackage{epstopdf}
\usepackage{algpseudocode}
\usepackage{scrextend}
\usepackage[T1]{fontenc}
\usepackage{bbm}
\usepackage{comment}
\usepackage{xspace}

\usepackage{tikz}
\usepackage{hyperref}  
\hypersetup{colorlinks=true,citecolor=blue,linkcolor=blue} 
\usetikzlibrary{arrows}
\usepackage[margin=1in]{geometry}
\graphicspath{{./figs/}}
\newcommand{\indep}{\rotatebox[origin=c]{90}{$\models$}}
\newcommand*{\centernot}{%
  \mathpalette\@centernot
}

\newenvironment{CompactItemize}{
\begin{list}{$\bullet$}{%
\setlength{\leftmargin}{12pt}
\setlength{\itemindent}{5pt}
\setlength{\topsep}{1pt}
\setlength{\itemsep}{-2pt}
}}
{\end{list}}

\def\@centernot#1#2{%
  \mathrel{%
    \rlap{%
      \settowidth\dimen@{$\m@th#1{#2}$}%
      \kern.5\dimen@
      \settowidth\dimen@{$\m@th#1=$}%
      \kern-.5\dimen@
      $\m@th#1\not$%
    }%
    {#2}%
  }%
}

\newcommand{\SSMatrix}{\text{\sc SSMatrix\xspace}}
\newcommand{\epsSSMatrix}{$\eps$-\text{\sc SSMatrix\xspace}}
\newcommand{\RecoverG}{\text{\sc RecoverG\xspace}}
\newcommand{\LatentsNEdges}{\text{\sc LatentsNEdges\xspace}}
\newcommand{\LatentsWEdges}{\text{\sc LatentsWEdges\xspace}}

\author{
    Raghavendra Addanki\thanks{UMass Amherst. \texttt{raddanki@cs.umass.edu}. Part of this work was done while the author was an intern at Amazon.}
  \and
  Shiva Prasad Kasiviswanathan\thanks{Amazon. \texttt{ kasivisw@gmail.com} }
  \and
  Andrew McGregor\thanks{UMass Amherst. \texttt{mcgregor@cs.umass.edu} }
  \and
  Cameron Musco\thanks{UMass Amherst. \texttt{cmusco@cs.umass.edu} }
}

\newcommand{\notindep}{\not\!\perp\!\!\!\perp}

\date{}
\title{Efficient Intervention Design for Causal Discovery with Latents}

\newtheorem{theorem}{Theorem}[section]
\newtheorem{lemma}[theorem]{Lemma}
\newtheorem{definition}[theorem]{Definition}

\newtheorem{proposition}[theorem]{Proposition}
\newtheorem{corollary}[theorem]{Corollary}

\newtheorem{assumption}[theorem]{Assumption}
\newtheorem{observation}[theorem]{Observation}

\newcommand{\lat}{L}
\newcommand{\Edg}{\mathcal{E}}

\newcommand{\pparagraph}[1]{\noindent \textbf{#1}}
\newcommand{\wh}{\widehat}
\newcommand{\wt}{\widetilde}

\newcommand{\eps}{\epsilon}

\newcommand{\norm}[1]{\left\lVert#1\right\rVert}
\renewcommand{\varepsilon}{\epsilon}
\renewcommand{\tilde}{\wt}
\renewcommand{\hat}{\wh}

\DeclareMathOperator{\OPT}{OPT}

\newcommand{\BBB}{{\mathcal B}}
\newcommand{\AAA}{{\mathcal A}}
\newcommand{\DDD}{{\mathcal D}}

\definecolor{mygreen}{RGB}{80,180,0}
\definecolor{b2}{RGB}{51,153,255}
\definecolor{mycy2}{RGB}{255,51,255}

\def \doo {\mathrm{do}}
\def \Pa {\mathrm{Pa}}
\def \e {\mathrm{e}}
\def \Anc {\mathrm{Anc}}
\def \GGG {\mathcal{G}}

\makeatletter
\newcommand*{\RN}[1]{\expandafter\@slowromancap\romannumeral #1@}
\makeatother
\usepackage{lineno}

\usepackage{etoolbox}

\makeatletter

\newcommand{\define}[4][ignore]{%
  \ifstrequal{#1}{ignore}{}{
  \@namedef{thmtitle@#2}{#1}}%
  \@namedef{thm@#2}{#4}%
  \@namedef{thmtypen@#2}{lemma}%
  \newtheorem{thmtype@#2}[theorem]{#3}%
  \newtheorem*{thmtypealt@#2}{#3~\ref{#2}}%
}

\newcommand{\state}[1]{%
  \@namedef{curthm}{#1}
  \@ifundefined{thmtitle@#1}{
  \begin{thmtype@#1}
    }{
  \begin{thmtype@#1}[\@nameuse{thmtitle@#1}]
  }
    \label{#1}
    \@nameuse{thm@#1}
  \end{thmtype@#1}
  \@ifundefined{thmdone@#1}{
  \@namedef{thmdone@#1}{stated}%
  }{}
}

\newcommand{\restate}[1]{%
  \@namedef{curthm}{#1}
  \@ifundefined{thmtitle@#1}{
    \begin{thmtypealt@#1}
    }{
  \begin{thmtypealt@#1}[\@nameuse{thmtitle@#1}]
  }
    \@nameuse{thm@#1}
  \end{thmtypealt@#1}
  \@ifundefined{thmdone@#1}{
  \@namedef{thmdone@#1}{stated}%
  }{}
}

\newcommand{\thmlabel}[1]{
  \@ifundefined{thmdone@\@nameuse{curthm}}{\label{#1}
    }{\tag*{\eqref{#1}}}
}
\makeatother

\begin{document}
\maketitle

\begin{abstract}

We consider recovering a causal graph in presence of latent variables, where we seek to minimize the cost of interventions used in the recovery process. We consider two intervention cost models: (1) a linear cost model where the cost of an intervention on a subset of variables has a linear form, and (2) an identity cost model where the cost of an intervention is the same, regardless of what variables it is on, i.e., the goal is just to minimize the number of interventions. Under the linear cost model, we give an algorithm to identify the ancestral relations of the underlying causal graph, achieving within a $2$-factor of the optimal intervention cost. This approximation factor can be improved to $1+\eps$ for any $\eps > 0$ under some mild restrictions. Under the identity cost model, we bound the number of interventions needed to recover the entire causal graph, including the latent variables, using a parameterization of the causal graph  through a special type of colliders. In particular, we introduce the notion of $p$-colliders, that are colliders between pair of nodes arising from a specific type of conditioning in the causal graph, and provide an upper bound on the number of interventions as a function of the maximum number of $p$-colliders between any two nodes in the causal graph.
\end{abstract}


\section{Introduction} \label{sec:intro}

Causality has long been a key tool in studying and analyzing various behaviors in fields such as genetics, psychology, and economics~\citep{pearl}. Causality also plays a pivotal role in helping us build systems that can understand the world around us, and in turn, in helping us understand the behavior of machine learning systems deployed in the real world.
Although the theory of causality has been around for more than three decades, for these reasons it has received increasing attention in recent years. 
In this paper, we study one of the fundamental problems of causality: {\em causal discovery}. In causal discovery, we want to learn all the causal relations existing between variables (nodes of the causal graph) of our system. It has been shown that, under certain assumptions, observational data alone only lets us recover \emph{the existence of a causal relationship} between two variables, but not the \emph{direction} of all relationships.
To recover the directions of causal edges, we use the notion of an \emph{intervention} described in the Structural Causal Models (SCM) framework introduced by~\citep{pearl}.

An intervention requires us to fix a subset of variables to a set of values, inducing a new distribution on the free variables. Such a system manipulation is generally expensive and thus there has been significant interest in trying to minimize the number of interventions and their cost in causal discovery.  
In a general cost model, intervening on any subset of variables has a cost associated with it, and the goal is to identify all causal relationships and their directions while minimizing the total cost of interventions applied. This captures the fact that some interventions are more expensive than others. For example, in a medical study, intervening on certain variables might be impractical or unethical.
In this work, we study two simplifications of this general cost model. In the {\em linear cost model}, each variable has an intervention cost, and the cost of an intervention on a subset of variables is the sum of costs for each variable in the set~\citep{icml17,neurips18}. In the {\em identity cost model}, every intervention has the same cost, regardless of what variables it contains and therefore minimizing intervention cost is the same as minimizing the number of interventions~\citep{neurips17}.

As is standard in the causality literature, we assume that our causal relationship graph satisfies the {\em causal Markov condition} and {\em faithfulness}~\citep{spirtes2000causation}. 
We assume that faithfulness holds both in the observational and interventional distributions following~\citep{hauser2014two}. As is common, we also assume that we are given access to an oracle that can check if two variables are independent, conditioned on a subset of variables. We discuss this assumption in more detail in Section \ref{sec:prelim}.
Unlike much prior work, we \emph{do not make} the { causal sufficiency} assumption: that there are no unobserved (or latent) variables in the system. Our algorithms apply to the causal discovery problem with the existence of latent variables.

\vspace{2mm}
\pparagraph{Results.} Our contributions are as follows. Let $\GGG$ be a causal graph on both observable variables $V$ and latent variables $\lat$. A directed edge $(u,v)$ in $\GGG$ indicates a causal relationship from $u$ to $v$. Let $G$ be the induced subgraph of $\GGG$ on the $n$ observable variables (referred to as observable graph). See Section~\ref{sec:prelim} for a more formal description.

\vspace{2mm}
\pparagraph{Linear Cost Model.} In the linear cost model, we give an algorithm that given $m = \Omega(\log n)$ where $n$ is the number of observed variables, outputs a set of $m$ interventions that can be used to recover all ancestral relations of the observable graph $G$.\!\footnote{As noted in Section~\ref{sec:lincost}, $m \geq \log n$ {is a lower bound for any solution.}} We show that cost of interventions generated by the algorithm is at most twice the cost of the optimum set of interventions for this task. 
Our result is based on a characterization that shows that generating a set of interventions sufficient to recover ancestral relations  is equivalent to designing a {\em strongly separating set system} (Def.~\ref{def:ssss}). We show how to design such a set system with at most twice the optimum cost based on a greedy algorithm that constructs intervention sets which includes a variable with high cost in the least number of sets possible.

In the special case when each variable has unit intervention
cost~\citep{hyttinen2013experiment} give an exact algorithm to recover ancestral relations in  $G$ with minimal total cost. Their algorithm is based on the Kruskal-Katona theorem in combinatorics \citep{kruskal1963number,katona1966separating}. We show that a modification of this approach yields a $(1+\epsilon)$-approximation algorithm in the general linear cost model for any $0< \eps \leq 1$, under mild assumptions on $m$ and the maximum intervention cost on any one variable.

The linear cost model was first considered in~\citep{icml17} and studied under the causal sufficiency (no latents) assumption.~\citep{neurips18} showed given the {\em essential graph} of a causal graph, 
the problem of recovering $G$ with optimal cost under the linear cost model is NP-hard. 
To the best of our knowledge, our result is the first to minimize intervention cost under the popular linear cost model in the presence of latents, and without the assumption of unit intervention cost on each variable.

We note that, while we give a 2-approximation for recovering ancestral relations in $G$, under the linear cost model, there seems to be no known  characterization of the optimal intervention sets needed to recover the entire causal graph $\mathcal{G}$, making it hard to design a good approximation here. Tackling this problem in the linear cost model is an interesting direction for future work.

\vspace{2mm}
\pparagraph{Identity Cost Model}. In the identity cost model, where we seek to just minimize the number of interventions, recovering ancestral relations in $G$ with minimum cost becomes trivial (see Section~\ref{sec:unweighted}).
Thus, in this case, we focus on algorithms that recover the causal graph $\GGG$ completely. 
We start with the notion of {\em colliders} in causal graphs~\citep{pearl}. Our idea is to parametrize the causal graph in terms of a specific type of colliders that we refer to as $p$-colliders (Def.~\ref{def:pcollider}). Intuitively, a node $v_k$ is $p$-collider for a pair of nodes $(v_i,v_j)$ if a) it is a collider on a path between $v_i$ and $v_j$ and b) at least one of the parents $v_i,v_j$ is a descendent of $v_k$. If the graph $\GGG$ has at most $\tau$ $p$-colliders between every pair of nodes, then our algorithm uses at most $O(n\tau \log n + n \log n)$  interventions.  We also present causal graph instances where any non-adaptive algorithm requires $\Omega(n)$  interventions.

{The only previous bound on recovering $\GGG$ in this setting utilized $O(\min\{d \log^2 n ,\ell\} + d^2 \log n)$ interventions where $d$ is the maximum (undirected) node degree and $\ell$ is the length of the longest directed path of the causal graph~\citep{neurips17}. Since we use a different parameterization of the causal graph, a direct comparison with this result is not always possible. We argue that a parameterization in terms of $p$-colliders is inherently more ``natural'' as it takes the directions of edges in $\GGG$ into account whereas the maximum degree does not. The presence of a single high-degree node can make the number of interventions required by existing work extremely high, even if the overall causal graph is sparse. In this case, the notion of $p$-colliders is a more global characterization of a causal graph. See Section~\ref{sec:experiments} for a more detailed discussion of different parameter regimes under which our scheme provides a better bound. We also experimentally show that our scheme achieves a better bound over~\citep{neurips17} in some popular random graph models.}

\subsection{Other Related Work}
Broadly, the problem of causal discovery has been studied under two different settings. In the first, one assumes \emph{causal sufficiency}, i.e., that there are no unmeasured (latent) variables. Most work in this setting focuses on recovering causal relationships based on just observational data. Examples include algorithms like {\em IC}~\citep{pearl} and {\em PC}~\citep{spirtes2000causation}.
Much work has focused on understanding the limitations and assumptions underlying these algorithms~\citep{hauser2014two, hoyer2009nonlinear, heinze2018causal, loh2014high, hoyer2009nonlinear, shimizu2006linear}. It is well-known, that to disambiguate a causal graph from its equivalence class, interventional, rather than just observational data is required~\citep{hauser2012characterization, eberhardt2007interventions, eberhardt2007causation}. In particular, letting $\chi(\mathcal G)$ be the chromatic number of $G$, $\Theta(\log \chi(\mathcal G))$ interventions are necessary and sufficient for recovery under the causal sufficiency assumption~\citep{hauser2014two}.  Surprising connections have been found~\citep{hyttinen2013experiment, katona1966separating, mao1984separating} between combinatorial structures and causality. Using these connections, much recent work has been devoted to minimizing intervention cost while imposing constraints such as sparsity or different costs for different sets of nodes~\citep{shanmugam2015learning, icml17, neurips18}.

In many  cases, causal sufficiency is too strong an assumption: it is often contested if the behavior of systems we observe can truly be attributed to measured variables~\citep{pearl2000causality, bareinboim2016causal}. In light of this, many algorithms avoiding the causal sufficiency assumption, such as IC$^*$~\citep{ic} and FCI~\citep{spirtes2000causation}, have been developed. The above algorithms only use observational data. However, there is a growing interest in optimal intervention design in this setting~\citep{silva2006learning, hyttinen2013discovering, parviainen2011ancestor}.  We contribute to this line of work, focusing on minimizing the intervention cost required to recover the full intervention graph, or its ancestral graph, in the presence of latents.


\section{Preliminaries} \label{sec:prelim}
\pparagraph{Notation.} Following the SCM framework introduced by \citet{pearl}, we represent the set of random variables of interest by $V \cup \lat$ where $V$ represents the set of endogenous (observed) variables that can be measured and $\lat$ represents the set of exogenous (latent) variables that cannot be measured. 
We define a directed graph on these variables where an edge corresponds to a causal relation between the corresponding variables. The edges are directed with an edge $(v_i,v_j)$ meaning that $v_i \rightarrow v_j$. As is common, we assume that all causal relations that exist between random variables in $V \cup \lat$ belong to one of the two categories : (i) $E \subseteq V \times V$ containing causal relations between the observed variables and (ii) $E_L \subseteq L \times V$ containing relations of the form $l \rightarrow v$ where $l \in \lat, v \in V$. Thus, the full edge set of our causal graph is denoted by $\Edg = E \cup E_L$. We also assume that every latent $l \in \lat$ influences exactly two observed variables i.e., $(l,u), (l,v) \in E_L$ and no other edges are incident on $l$ following~\citep{neurips17}. We let $\GGG = \GGG(V \cup \lat, \Edg)$  denote the entire causal graph and refer to $G=G(V, E)$ as the {\em observable graph}.

Unless otherwise specified a path between two nodes is a undirected path. 
For every observable $v \in V$, let the parents of $v$ be defined as $\Pa(v) = \{ w \mid w \in V \text{ and } (w,v) \in E \}$. For a set of nodes $S \subseteq V$, $\Pa(S) = \cup_{v \in S} \Pa(v)$. If $v_i, v_j \in V$, we say $v_j$ is a descendant of $v_i$ (and $v_i$ is an ancestor of $v_j$) if there is a directed path from $v_i$ to $v_j$. $\Anc(v) = \{ w \mid w \in V \text{ and } v \text{ is a descendant of }w \}$. We let $\Anc(G)$ denote the \emph{ancestral graph}\footnote{We note that the term {\em ancestral graph} has also been previously used in a different context, see e.g.,~\citep{richardson2002ancestral}.}  of $G$ where an edge $(v_i, v_j) \in \Anc(G)$ if and only if there is a directed path from $v_i$ to $v_j$ in $G$. One of our primary interests is in recovering $\Anc(G)$ using a minimal cost  set of interventions.

Using Pearl's do-notation, we represent an intervention on a set of variables $S \subseteq V$ as $\doo(S = s)$ for a value $s$ in the domain of $S$ and the joint probability distribution on $V \cup \lat$ conditioned on this intervention by $\Pr[\cdot \mid do(S)]$. 

We assume that there exists an oracle that answers queries such as \emph{``Is $v_i$ independent of $v_j$ given $Z$ in the interventional distribution $\Pr[\cdot \mid \doo(S = s)]$?''} 

\begin{assumption}[Conditional Independence (CI)-Oracle] \label{assum:CIoracle}
Given any  $v_i, v_j \in V$ and $Z, S \subseteq V$ we have an oracle that tests whether $v_i \indep v_j \mid Z, {\doo(S = s)}$. 
\end{assumption}
Such conditional independence tests have been widely investigated with sublinear (in domain size) bounds on the sample size needed for implementing this oracle~\citep{canonne2018testing,zhang2011kernel}. 

\paragraph{Intervention Cost Models.}
We study the causal discovery problem under two cost models:

\begin{enumerate}
    \item \textit{Linear Cost Model.} In this model, each node $v \in V$ has a different cost $c(v) \in \mathbb{R}^+$ and the cost of intervention on a set $S \subset V$ is defined as $\sum_{v \in S} c(v)$ (akin to~\citep{neurips18}). That is, interventions that involve a larger number of, or more costly nodes, are more expensive. Our goal is to find an intervention set $\mathcal{S}$ minimizing $\sum_{S \in \mathcal{S}} \sum_{v \in S} c(v)$. We constrain the number of interventions to be upper bounded by  some budget $m$. Without such a bound, we can observe that for ancestral graph recovery, the optimal intervention set is $\mathcal{S} = \{\{v_1\},\{v_2\},\ldots,\{v_n\}\}$ with cost $\sum_{v \in V} c(v)$ as intervention on every variable is necessary, {as we need to account for the possibility of latent variables (See Lemma~\ref{lem:lb-strong} for more details)}. The optimality of $\mathcal{S}$ here follows from a characterization of any feasible set system we establish in Lemma~\ref{lem:lb-strong}.  
    
    \item \textit{Identity Cost Model.} As an intervention on a set of variables requires controlling the variables, and generating a new distribution, we want to use as few interventions as possible. In this cost model, an intervention on any set of observed variables has \textit{unit cost} (no matter how many variables are in the set). We assume that for any intervention, querying the CI-oracle comes free of cost. This model is akin to the model studied in~\cite{neurips17}.
\end{enumerate}

\pparagraph{Causal Discovery Goals.}
We will study two variations of the causal discovery problem. In the first, we aim to recover the ancestral graph $\Anc(G)$, which contains all the causal ancestral relationships between the observable variables $V$. In the second, our goal is to recover all the causal relations in $\Edg$, i.e., learn the entire causal graph $\GGG(V \cup \lat, \Edg)$. 
We aim to  perform both tasks using a set of intervention sets $\mathcal{S} = \{S_1,\ldots, S_m\}$ (each $S_i \subseteq V$) with minimal cost, with our cost models defined above.

For ancestral graph recovery, we will leverage a simple characterization of when a set of interventions $\mathcal{S} = \{S_1,\ldots, S_m\}$ is sufficient to recover $\Anc(G)$. In particular, $\mathcal{S}$ is sufficient if it is a \emph{strongly separating set system}~\citep{neurips17}. 

\begin{definition}[Strongly Separating Set System] \label{def:ssss}
A collection of subsets $\mathcal{S} = \{S_1, \cdots, S_m\}$ of the ground set $V$ is a strongly separating set system if for every distinct $u,v\in V$ there exists $S_i$ and $S_j$ such that $u\in S_i\setminus S_j$ and $v\in S_j\setminus S_i$.

\end{definition}

Ancestral graph recovery using a strongly separating set system is simple: we intervene on each of the sets $S_1,\dots,S_m$. Using CI-tests we can identify for every pair of $v_i$ and $v_j$, if there is a path from $v_i$ to $v_j$ or not in $G$ using the intervention corresponding to $S \in \mathcal{S}$ with $v_i \in S \text{ and } v_j \notin S$. We add an edge to $\Anc(G)$ if the test returns dependence. Finally, we take the transitive closure and output the resulting graph as $\Anc(G)$. In Lemma \ref{lem:lb-strong}, we show that in fact being strongly separating is \emph{necessary} for any set of interventions to be used to identify $\Anc(G)$.

\section{Linear Cost Model} \label{sec:lincost}
We begin with our results on recovering the  ancestral graph $\Anc(G)$ in the linear cost model. Recall that, given a budget of $m$ interventions, our objective is to  find a set of interventions $\mathcal{S} = \{ S_1, S_2, \cdots S_m \}$ that can be used to identify $\Anc(G)$ while minimizing $\sum_{S \in \mathcal{S}} \sum_{v \in S} c(v)$.

As detailed in Section~\ref{sec:prelim}, a strongly separating set system is sufficient to recover the ancestral graph. We show that it also necessary: a set of interventions to discover $\Anc(G)$ must be a strongly separating set system (Definition~\ref{def:ssss}). See proof in Appendix~\ref{app:lincost}.

\begin{lemma}
\label{lem:lb-strong}
Suppose $\mathcal{S} = \{ S_1, S_2, \cdots, S_m \}$ is a collection of subsets of $V$. For a given causal graph $G$ if $\Anc(G)$ is recovered using CI-tests by intervening on the sets $S_i \in \mathcal{S}$. Then, $\mathcal{S}$ is a strongly separating set system.
\end{lemma}

Given this characterization, the problem of constructing the ancestral graph $\Anc(G)$ with minimum linear cost {\em reduces} to that of constructing a strongly separating set system with minimum cost. In developing our algorithm for finding such a set system, it will be useful to represent a set system by a binary  matrix, with rows corresponding to observable variables $V$ and columns corresponding to interventions (sets $S_1,\ldots, S_m$).

\begin{definition}[Strongly Separating Matrix] \label{def:SSM}
Matrix $U \in \{0,1\}^{n \times m}$ is a \emph{strongly separating matrix} if $\forall i,j \in [n]$ $\text{there exists } k, k' \in [m] \text{ such that } U(i, k) = 1, U(j, k) = 0$ and $U(i, k') = 0, U(j, k') = 1$.
\end{definition}
Note that given a strongly separating set system $\mathcal{S}$, if we let $U$ be the matrix where $U(i,k)=1$ if $v_i \in S_k$ and $0$ otherwise, $U$ will be a strongly  separating matrix. The other direction is also true. 
Let $U(j)$ denote the $j$th row of $U$.
Using Definition~\ref{def:SSM} and above connection between recovering $\Anc(G)$ and strongly separating set system, we can reformulate the problem at hand as:
\begin{align} 
&\text{min}_U \sum_{j=1}^n c(v_j)\cdot \| U(j) \|_1  \label{eqn:obj} \\
&\text{ s.t. } U \in \{0,1\}^{n \times m} \text{ is a strongly separating matrix. }\nonumber
\end{align}
We can thus view our problem as finding an assignment of vectors in $\{0,1\}^m$ (i.e., rows of $U$) to nodes in $V$ that minimizes \eqref{eqn:obj}. 
Throughout, we will call $\norm{U(j)}_1$ the \emph{weight} of row $U(j)$, i.e., the number of $1$s in that row.
It is easy to see that $m\ge \log n$ is necessary for a feasible solution to exist as each row must be distinct.

We start by giving a 2-approximation algorithm for~\eqref{eqn:obj}. In Section~\ref{sec:kruskal-katona}, we show how to obtain an improved approximation under certain assumptions.

\subsection{$2$-approximation Algorithm} \label{sec:2approx} 
In this section, we present an algorithm (Algorithm~\SSMatrix) that constructs a strongly separating matrix (and a corresponding intervention set) which minimizes~\eqref{eqn:obj} to within a $2$-factor of the optimum. Missing details from section are collected in Appendix~\ref{app:2approx}.

\newcommand{\rvline}{\hspace*{-\arraycolsep}\vline\hspace*{-\arraycolsep}}

\vspace{2mm}
\pparagraph{Outline.}  Let $U_{\OPT}$ denote a strongly separating matrix minimizing~\eqref{eqn:obj}. Let $c_{\OPT} = \sum_{j=1}^n c(v_j) \|U_{\OPT}(j)\|_1$ denote the objective value achieved by this optimum $U_{\OPT}$. We start by relaxing the constraint on $U$ so that it does \emph{not need to be strongly separating}, but just must have unique rows, where none of the rows is all zero. In this case, we can optimize \eqref{eqn:obj} very easily. We simply take the rows of $U$ to be the $n$ unique binary vectors in $\{0,1\}^m \setminus \{0^m\}$ with lowest weights. That is, $m$ rows will have weight $1$, $\binom{m}{2}$ will have weight $2$, etc. We then assign the rows to the nodes in $V$ in descending order of their costs. So the $m$ nodes with the highest costs will be assigned the weight $1$ rows, the next $\binom{m}{2}$ assigned weight $2$ rows, etc. The cost  of this assignment is only lower than $c_{\OPT}$, as we have only  relaxed the constraint in \eqref{eqn:obj}.

We next convert this relaxed solution into a valid strongly separating matrix. Given $m + \log n$ columns, we can do this easily.
Since there are $n$ nodes, in the above assignment, all rows will have weight at most $\log n$. Let $\bar U \in \{0,1\}^{m + \log n}$ have its first $m$ columns equal to those of $U$. Additionally, use the last $\log n$ columns as `row weight indicators': if $\norm{U(j)}_1 = k$ then set $\bar U(j,m+k) = 1$.  We can see that $\bar U$ is a strongly  separating matrix. If two rows have different weights $k,k'$ in $\bar U$, then the last $\log n$ columns ensure that they satisfy the strongly separating condition. If they  have the same weight in $\bar U$, then they already satisfy the condition, as to be unique in $U$ they must have a at least 2 entries on which they differ.

To turn the above idea into a valid approximation algorithm that outputs $\bar U$ with just $m$ (not $m + \log n$) columns, we argue that we can `reserve' the last $\log n$ columns of $\bar U$ to serve as weight indicator columns. We are then left with just $m-\log n$ columns to work with. Thus we can only assign $m-\log n$ weight $1$ rows, $\binom{m-\log n}{2}$ weight 2 rows, etc. Nevertheless, if $m \geq { \gamma} \log n$ (for a constant $\gamma > 1$), this does not affect the assignment much: for any $i$ we can still `cover' the $\binom{m}{i}$ weight $i$ rows in $U$ with rows of weight $\le 2i$. Thus, after accounting for the weight indicator columns, each weight $k$ row in $U$ has weight $\le 2k+1$ in $\bar U$. Overall, this gives us a 3-approximation algorithm: when $k$ is $1$ the weight of a row may become as large as $3$.

To improve the approximation to a 2-approximation we \emph{guess} the number of weight $1$ vectors $a_1$ in the optimum solution $U_{\OPT}$ and assign the $a_1$ highest cost variables  to weight $1$ vectors, achieving optimal cost for these variables. There are $O(m)$ possible values for $a_1$ and so trying all guesses is still efficient. We then apply  our approximation algorithm to the remaining $m-a_1$ available columns of $U$ and $n - a_1$ variables. Since no variables  are assigned weight $1$ in this set, we achieve a tighter $2$-approximation using our approach. The resulting matrix has the form:
\[
U=\begin{pmatrix}
\mathbb{I}_{a_1} & \rvline &  \o0  & \rvline &   0  \\   \hline
0 & \rvline &  C_1 &  \rvline &
M_1  \\ \hline 
0 & \rvline &  C_2 &  \rvline &
M_2 \\ \hline 
\vdots & \rvline & \vdots &\rvline & \vdots
\end{pmatrix}
\]
 where $\mathbb{I}_{a_1}$ is the $a_1\times a_1$ identity matrix, the rows of $C_w$ are all weight $w$ binary vectors of length $m-\log n-a_1$, and the rows of $M_w$ are length $\log n$ binary vectors with 1's in the $w$th column. The entire approach is presented in Algorithm~\SSMatrix~and a proof of the approximation bound in Theorem~\ref{thm:ss_2approx} is present in Appendix~\ref{app:2approx}. 
\begin{algorithm}[t]
\begin{small}
\caption{\SSMatrix $(V,m)$}
\label{alg:ss_matrix}
\begin{algorithmic}[1]

\State  $c_{U_{min}} \leftarrow \infty$
\For{$a_1 \in \{ 0, 1, \cdots, {2m}/{3} \}$}
\State  $U \in \{0,1\}^{n \times m}$ be initialized with all zeros
\State  Assign the highest cost $a_1$ nodes with unit weight vectors such that $U(i,i) = 1$ for $i \leq a_1$
\State  Set $m' \leftarrow m - a_1$
\State  Mark all vectors of weight \textit{at least} $1$ in $\{0,1\}^{m'-\log n}$ as available 
\For{unassigned $v_i \in V$ (in decreasing order of cost)}
\State Set $U(i, (a_1+1):m-\log n)$ to smallest available weight vector in $\{ 0 , 1 \}^{m'-\log n}$ and make this vector unavailable. Let the weight of the assigned vector be $k$
\State Set `row weight indicator' $U(i, m'-\log n + k) = 1$  

\EndFor 
\State  Compute cost of objective for $U$ be $c_{U}$
\If{$c_{U}$ < $c_{U_{min}}$} 
\State  $c_{U_{min}} \leftarrow {c}_{U}, U_{min} \leftarrow U$
\EndIf
\EndFor
\State  Return $U_{min}$
\end{algorithmic}
\end{small}
\end{algorithm}
\setlength{\textfloatsep}{2pt}

\begin{theorem}
\label{thm:ss_2approx}
Let $m \geq \gamma \log n$ for constant $\gamma>1$ and $U$ be the strongly separating matrix returned by~\SSMatrix.\!\footnote{In our proof, $\gamma = 66$ but this can likely be decreased.} Let $c_U = \sum_{j=1}^n c(v_j)\, \| U(j) \|_1$. Then, $c_U \leq 2 \cdot c_{\OPT}$, where $c_{\OPT}$ is the objective value associated with optimum set of interventions corresponding to $U_{\OPT}$.
\end{theorem}

Using the interventions from the matrix $U$ returned by Algorithm~\SSMatrix, we obtain
a cost within twice the optimum for recovering $\Anc(G)$.

\subsection{$(1+\epsilon)$-approximation Algorithm}
\label{sec:kruskal-katona}
In~\citep{hyttinen2013experiment}, the authors show how to construct a collection $\mathcal{A}$ of $m$ strongly separating intervention sets with minimum average set size, i.e., $\sum_{A \in \mathcal{A}} |A|/m$. This is equivalent to minimizing the objective~\eqref{eqn:obj} in the linear cost model when the cost of intervening on any node equals $1$.  
In this section, we analyze an adaptation of their algorithm to the general linear cost model, and obtain a $(1+\epsilon)$-approximation for any given $0 < \epsilon \leq 1$, an improvement over the $2$-approximation of Section \ref{sec:2approx}. Our analysis requires mild restrictions on the number of interventions  and an upper bound on the maximum cost. The algorithm will not depend on $\epsilon$ but these bounds will.
Missing details from this section are collected in Appendix~\ref{app:kruskal-katona}.

\vspace{2mm}
\pparagraph{Algorithm~\epsSSMatrix~Outline}. 
The famous Kruskal-Katona theorem in combinatorics forms the basis of the scheme presented in~\citep{hyttinen2013experiment} for minimizing the average size of the intervention sets. To deal with with varying costs of node interventions, we augment this approach with a greedy strategy. Let $\mathcal{A}$ denote a set of $m$ interventions sets over the nodes $\{v_1, v_2 \cdots, v_n \}$ obtained using the scheme from~\citep{hyttinen2013experiment}. Construct a strongly separating matrix $\tilde{U}$ from $\mathcal{A}$ with $\tilde{U}(i,j) = 1$ iff $v_i \in A_j$ for $A_j \in \mathcal{A}$. Let $\zeta$ denote the ordering of rows of $\tilde{U}$ in the increasing order of weight. Our Algorithm~\epsSSMatrix~outputs the strongly separating matrix $U$ where, for every $i \in [n]$, $U(i) = \tilde{U}(\zeta(i))$ and the $i$th row of $U$ corresponds to the node with $i$th largest cost.

Let $c_{max} = {\text{max}_{v_i \in V} \,c(v_i)}/{\text{min}_{v_i \in V} \,c(v_i)}$ be the ratio of maximum cost to minimum cost of nodes in $V$. For ease of analysis, we assume that the cost of any node is least $1$.

\begin{theorem}\label{thm:approx_1+e}
Let $U$ be the strongly separating matrix returned by~\epsSSMatrix. 
If $c_{max} \leq \frac{\epsilon n}{3 \binom{m}{t}}$ for $0< \eps \leq 1$ where $\binom{m}{k-1}<n\leq \binom{m}{k}$ and $t=\lfloor k-\epsilon k/3 \rfloor$, then, 
\[c_U:= \sum_{j=1}^n c(v_j)\, \| U(j) \|_1 \leq (1+\eps) \cdot c_{\OPT} \ ,\] where $c_{\OPT}$ is the objective value associated with optimum set of interventions corresponding to $U_{\OPT}$.
\end{theorem}

\begin{proof}
Suppose the optimal solution $U_{\OPT}$ includes $a^*_q$ vectors of weight $q$. Let $S$ be the $a^*_1+a^*_2+\ldots +a^*_t$ nodes with highest cost in $U_{\OPT}$. 
Since $a^*_q\leq \binom{m}{q}$, it immediately follows  that $|S|\leq \sum_{i=q}^{t} \binom{m}{q}$.  However, a slightly tighter analysis (see Lemma~\ref{lem:LYM_ineq}) implies $|S| \leq \binom{m}{t}.$
 Let $c_{\OPT}(S)$ be the total contribution of the nodes in $S$ to $c_{\OPT}$. Let $c_{U}(S)$ denote the sum of contribution of the nodes in $S$ to $c_U$ for the matrix $U$ returned by~\epsSSMatrix.  Let $\bar{k}_{|S|}$ and $\bar{k}_n$ be the average of the smallest $|S|$ and $n$ respectively  of the vector weights assigned by the algorithm. It is easy to observe that $\bar{k}_{|S|} \leq \bar{k}_n$.
\begin{align*}
  c_U(S) &= \sum_{v_i \in S} c(v_i) \norm{U(i)}_1  \leq  c_{max} \sum_{v_i \in S} \norm{U(i)}_1\\
  &= c_{max} \bar{k}_{|S|} |S|  \leq c_{max} \bar{k}_{|S|} \binom{m}{t} \leq   \eps \bar{k}_{|S|} n/3.
\end{align*}

As every node in $V\setminus S$ receives weight at least $t = k-\epsilon k/3$ in $U_{\OPT}$ and at most $k$ in $U$ returned by~\epsSSMatrix, we have $c_{U}(V\setminus S) \leq \frac{c_{\OPT}(V\setminus S)}{1-\epsilon/3}.$ 
Now, we give a lower bound on the cost of the optimum solution $c_{\OPT}(V)$. {We know that when costs of all the nodes are $1$, then~\epsSSMatrix achieves optimum cost denoted by $c'_{\OPT}(V)$} (see Appendix~\ref{app:kruskal-katona} for more details). As all the nodes of $V$ have costs more than $1$, we have: 

\begin{align*}
c_{\OPT}(V) &\geq c'_{\OPT}(V) = \bar{k}_n \cdot n\geq \bar{k}_{|S|} \cdot n.    
\end{align*}
Hence,
\[
\frac{c_U(V)}{c_{\OPT}(V)} 
\leq 
\frac{c_U(S)}{\bar{k}_{|S|} n} + \frac{c_U(V\setminus S)}{c_{\OPT}(V\setminus S)}
\leq  
\frac{\epsilon}{3} + \frac{1}{1-\epsilon/3}
\leq 
1+ \epsilon.
\]
This completes the proof.

\end{proof}
By bounding the binomial coefficients in Thm. \ref{thm:approx_1+e}, we obtain the following somewhat easier to interpret corollary:
\begin{corollary}
\label{cor:maxcost}

If $c_{max}\leq (\epsilon/6) n^{\Omega(\epsilon)}$ and either a) $n^{\epsilon/6} \geq m\geq (2\log_2 n)^{c_1}$ for some constant $c_1>1$ or b) $4\log_2 n\leq m\leq c_2 \log_2 n$ for some constant $c_2$ then the Algorithm~\epsSSMatrix~returns an $(1+\epsilon)$-approximation. 
\end{corollary}

\section{Identity Cost Model} \label{sec:unweighted}
In this section, we consider the identity cost model, where the cost of intervention for any subset of variables is the same. Our goal is to construct the entire causal graph $\GGG$, while minimizing the number of interventions.  Our algorithm is based on parameterizing the causal graph based on a specific type of collider structure. 
 
Before describing our algorithms, we recall the notion of $d$-separation and introduce this specific type of colliders that we rely on. Missing details from section are collected in Appendix~\ref{app:unweighted}.

\vspace{2mm}

\pparagraph{Colliders.} Given a causal graph $\GGG(V \cup \lat, \Edg)$, let $v_i, v_j \in V$ and a set of nodes $Z \subseteq V$. We say $v_i$ and $v_j$ are {\em $d$-separated} by $Z$ if and only if every undirected path $\pi$ between $v_i$ and $v_j$ is {\em blocked} by $Z$. A path $\pi$ between $v_i$ and $v_j$ is blocked by $Z$ if at least one of the following holds.
\begin{description}
\item[Rule 1:] $\pi$ contains a node $v_k \in Z$ such that the path $\pi =v_i \cdots \rightarrow v_k \rightarrow \cdots v_j$ or 
$v_i \cdots \leftarrow v_k \leftarrow \cdots v_j$.
\item[Rule 2:] $\pi = v_i \cdots \rightarrow v_k \leftarrow \cdots v_j$ contains a node $v_k$ and both $v_k \notin Z$ and no descendant of $v_k$ is in $Z$. 
\end{description}

\begin{lemma}{\citep{pearl}}
If $v_i$ and $v_j$ are $d$-separated by $Z$, then $v_i  \indep v_j \mid Z$.
\end{lemma}
For the path $\pi = v_i \cdots \rightarrow v_k \leftarrow \cdots v_j$ between $v_i$ and $v_j$, $v_k$ is called a {\em collider} as there are two arrows pointing towards it. We say that $v_k$ is a collider for the pair $v_i$ and $v_j$, if there exists a path between $v_i$ and $v_j$ for which $v_k$ is a collider. As shown by Rule 2, colliders play an important role in $d$-separation. We give a more restrictive definition for colliders that we will rely on henceforth. 
\begin{figure}[h]
\centering
\includegraphics[scale=0.6]{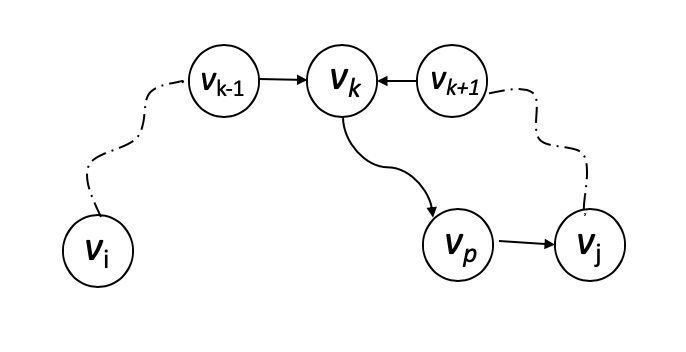}
\caption{$v_k$ is a $p$-collider for $v_i,v_j$ as it has a path to $v_p$, a parent of $v_j$.}
\label{fig:pcol}
\end{figure}
\begin{definition}[$p$-\textbf{colliders}] \label{def:pcollider}
Given a causal graph $\GGG(V \cup \lat, E \cup E_L)$. Consider $ v_i, v_j \in V$ and $v_k \in V$. We say $v_k$ is a $p$-collider for the pair $v_i$ and $v_j$, if there exists a path $v_i \cdots \rightarrow v_k \leftarrow \cdots v_j$ in $\GGG$ and either $v_k \in \Pa(v_i) \cup \Pa(v_j)$ or has at least one descendant in $\Pa(v_i) \cup \Pa(v_j)$. Let $P_{ij} \subset V$ denote all the $p$-colliders between $v_i$ and $v_j$.
\end{definition}
\noindent {Intervening on p-colliders essentially breaks down all the {\em primitive inducing paths}. Primitive inducing paths are those whose endpoints cannot be separated by any conditioning~\citep{richardson2002ancestral}. Now, between every pair of observable variables, we can define a set of $p$-colliders as above. Computing $P_{ij}$ for the pair of variables $v_i$ and $v_j$ explicitly requires the knowledge of $\GGG$, however as we show below we can use randomization to overcome this issue.} The following parameterization of a causal graph will be useful in our discussions.
\begin{definition} [$\tau$-causal graph]
A causal graph $\GGG(V \cup \lat, \Edg)$ is a $\tau$-causal graph if for every pair of nodes in $V$ the number of $p$-colliders is at most $\tau$, i.e., $v_i,v_j \in V$ ($i \neq j$), we have $|P_{ij}| \leq \tau$.
\end{definition}
Note that every causal graph is at most $n-2$-causal. In practice, we expect $\tau$ to be significantly smaller. Given a causal graph $\GGG$, it is easy to determine the minimum values of $\tau$ for which it is $\tau$-causal, as checking for $p$-colliders is easy. Our algorithm recovers $\GGG$ with number of interventions that grow as a function of $\tau$ and $n$.  
\vspace{2mm}

\pparagraph{Outline of our Approach.} Let $\GGG$ be a $\tau$-causal graph. 
As in~\citep{neurips17}, we break our approach into multiple steps. Firstly, we construct the ancestral graph $\Anc(G)$ using the strongly separating set system (Definition~\ref{def:ssss}) idea detailed in Section~\ref{sec:prelim}. For example, a strongly
separating set system can be constructed with $m =
2 \log n$ interventions by using the binary encoding of
the numbers $1,\cdots,n$~\citep{neurips17}.
After that the algorithm has two steps. In the first step, we recover the observable graph $G$ from $\Anc(G)$. In the next step, after obtaining the observable graph, we identify all the latents $\lat$ between the variables in $V$ to construct $\GGG$. In both these steps, an underlying idea is to construct intervention sets  with the aim of making sure that all the $p$-colliders between every pair of nodes is included in at least one of the intervention sets.  As we do not know the graph $\GGG$, we devise randomized strategies to hit all the $p$-colliders, whilst ensuring that we do not create a lot of interventions.

A point to note is that, we design the algorithms to achieve an overall success probability of $1-O(1/n^2)$, however, the success probability can be boosted to any $1-O(1/n^c)$ for any constant $c$, by just adjusting the constant factors (see for example the proof of Lemma~\ref{lem:graphwhp}). Also for simplicity of discussion, we assume that we know $\tau$. However as we discuss in Appendix~\ref{app:unweighted} this assumption can be easily removed with an additional $O(\log \tau)$ factor.

\subsection{Recovering the Observable Graph}

$\Anc(G)$ encodes all the ancestral relations on observable variables $V$ of the causal graph $G$. To recover $G$ from $\Anc(G)$, we want to differentiate whether $v_i \rightarrow v_j$ represents an edge in $G$ or a directed path going through other nodes in $G$. We use the following observation, if $v_i$ is a parent of $v_j$, the path $v_i \rightarrow v_j$ is never blocked by any conditioning set $Z \subseteq V \setminus \{ v_i \}$. If $v_i \not \in \Pa(v_j)$, then we show that we can provide a conditioning set $Z$ in some interventional distribution $S$ such that $v_i \indep v_j \mid Z, \doo(S)$. For every pair of variables that have an edge in $\Anc(G)$, we design conditioning sets in Algorithm~\ref{alg:recoverG} that blocks all the paths between them.

Let $v_i \in \Anc(v_j) \setminus \Pa(v_j)$. We argue that conditioning on  $\Anc(v_j) \setminus \{ v_i \}$ in $\doo(v_i \cup P_{ij})$ blocks all the paths from $v_i$ to $v_j$. The first simple observation, from $d$-separation is that if we take a path that has no $p$-colliders between $v_i$ to $v_j$ (a $p$-collider free path) then it is blocked by conditioning on $\Anc(v_j) \setminus \{ v_i \}$ i.e., $v_i \indep v_j \mid \Anc(v_j) \setminus \{ v_i \}$.

The idea then will be to intervene on colliders $P_{ij}$ to remove these dependencies between $v_i$ and $v_j$ as shown by the following lemma.
\begin{lemma}
\label{lem:indep}
Let $v_i \in Anc(v_j)$. $v_i \indep v_j \mid \doo(v_i \cup P_{ij}), \Anc(v_j) \setminus \{ v_i \}$ iff $v_i \not \in \Pa(v_j)$.
\end{lemma}

From Lemma~\ref{lem:indep}, we can recover the edges of the observable graph $G$ provided we know the $p$-colliders between every pair of nodes. However, since the set of $p$-colliders is unknown without the knowledge of $\GGG$, we construct multiple intervention sets by independently sampling every variable with some probability. This ensures that there exists an intervention set $S$ such that $\{ v_i \} \cup P_{ij} \subseteq S$ and $v_j \not \in S$ with high probability.

Formally, let $A_t \subseteq V$ for $t \in \{1, 2, \cdots, 72\tau'\log n\}$ be constructed by including  every variable $v_i \in V$ with probability $1-1/{\tau'}$ where $\tau' = \max\{\tau,2\}$. Let $\AAA_\tau = \{A_1,\cdots,A_{72\tau' \log n}\}$ be the collection of the set $A_t$'s. Algorithm~\ref{alg:recoverG} uses the interventions in $\AAA_\tau$. 

\begin{algorithm}
\begin{small}
\caption{\RecoverG($\Anc(G), \AAA_\tau$)}
\label{alg:recoverG}
    \begin{algorithmic}[1]
     \State  $E = \phi$
	\For{$v_i \rightarrow v_j$ in $\Anc(G)$}
		 \State  Let $\mathcal{A}_{ij} = \{ A \mid A \in \AAA_\tau$ such that $v_i \in A, v_j \not \in A \}$
		\If{$\forall$ $A \in \mathcal{A}_{ij}$,$v_i \notindep v_j \mid \Anc(v_j) \setminus \{ v_i \}, \doo(A)$  }
			 \State  $E = E \cup \{ (v_i, v_j) \}$
		\EndIf
	\EndFor
     \State  \textbf{return} $E$
\end{algorithmic}
\end{small}
\end{algorithm}

\begin{proposition}
\label{prop:obs}
Let $\GGG(V \cup L, E \cup E_L)$ be a $\tau$-causal graph with observable graph $G(V,E)$. There exists a procedure to recover the observable graph using $O(\tau \log n + \log n)$ many interventions with  probability at least $1-1/n^2$.
\end{proposition}

\pparagraph{Lower Bound.} Complementing the above result, the following proposition gives a lower bound on the number of interventions by providing an instance of a $O(n)$-causal graph such that any non-adaptive algorithm requires $\Omega(n)$ interventions for recovering it. The lower bound comes because of the fact that the algorithm cannot rule out the possibility of latent.

\vspace{2mm}
\begin{proposition}
\label{prop:lb}
There exists a graph causal $\GGG(V \cup \lat, E \cup E_L)$ such that every non-adaptive algorithm requires $\Omega(n)$ many interventions to recover even the observable graph $G(V,E)$ of $\GGG$. 
\end{proposition}

\subsection{Detecting the Latents}

We now describe algorithms to identify latents that effect the observable variables $V$ to learn the entire causal graph $\GGG(V \cup \lat, E \cup E_L)$. We start from the observable graph $G(V, E)$ constructed in the previous section.  Our goal will be to use the fact that $\GGG$ is a $\tau$-causal graph, which means that $|P_{ij}| \leq \tau$ for every pair $v_i, v_j$. Since we assumed that each latent variable (in $\lat$) effects at most two observable variables (in $V$), we can split the analysis into two cases: a) pairs of nodes in $G$ without an edge (non-adjacent nodes) and b) pairs of nodes in $G$ with a direct edge (adjacent). In Algorithm~\LatentsNEdges (Appendix~\ref{app:unweighted}), we describe the algorithm for identifying the latents effecting pairs of non-adjacent nodes. The idea is to block the paths by conditioning on parents and intervening on $p$-colliders. We use the observation that for any non-adjacent pair $v_i, v_j$ an intervention on the set $P_{ij}$ and conditioning on the parents of $v_i$ and $v_j$ will make $v_i$ and $v_j$ independent, unless there is a latent between them.

\begin{proposition} \label{prop:1}
Let $\GGG(V \cup L, E \cup E_L)$ be a $\tau$-causal graph with observable graph $G(V,E)$.
Algorithm~\LatentsNEdges with $O(\tau^2 \log n + \log n)$ interventions recovers all latents effecting pairs of non-adjacent nodes in the observable graph $G$ with probability at least $1 - 1/n^2$.
\end{proposition}

\pparagraph{Latents Affecting Adjacent Nodes in $G$.}  Suppose we have an edge $v_i \rightarrow v_j$ in $G(V, E)$ and we want to detect whether there exists a latent $l_{ij}$ that effects both of them. Here, we cannot block the edge path $v_i \rightarrow v_j$ by conditioning on any $Z \subseteq V$ in any given interventional distribution $\doo(S)$ where $S$ does not contain $v_j$. However, intervening on $v_j$ also disconnects $v_j$ from its latent parent. Therefore, CI-tests are not helpful. Hence, we make use of another test called \emph{do-see} test~\citep{neurips17}, that compares two probability distributions. We assume there exists an oracle that answers whether two distributions are the same or not. This is a well-studied problem with sublinear (in domain size) bound on the sample size needed for implementing this oracle~\citep{chan2014optimal}.

\begin{algorithm}[!t]
\begin{small}
\caption{LatentsWEdges($\GGG(V \cup \lat, E \cup E_L), \BBB_\tau$)}
\label{alg:latentWE}
    \begin{algorithmic}[1]
     \State  Consider the edge $ v_i \rightarrow v_j \in {E}$.
     \State  Let $\mathcal{B}_{ij} = \{ B\setminus \{v_i\} \mid B \in \BBB_\tau \text{~s.t.~} v_i \in B, v_j \not \in B \}$
    \If{$\forall B\in \mathcal{B}_{ij}, \Pr[ v_j \mid v_i,  \Pa(v_j), \doo(\Pa(v_i) \cup B)]  \neq \Pr[ v_j \mid \Pa(v_j), \doo(\{v_i\} \cup \Pa(v_i) \cup B)]$}
     \State  $\lat \leftarrow \lat \cup {l_{ij}}, E_L \leftarrow E_L \cup \{ (l_{ij}, v_i), (l_{ij}, v_j) \}$
    \EndIf
     \State  \textbf{return} $\GGG(V \cup \lat, E \cup E_L)$    
\end{algorithmic}
\end{small}
\end{algorithm}

\begin{assumption}[Distribution Testing (DT)-Oracle] \label{assum:DToracle}
Given any $v_i, v_j \in V$ and $Z, S \subseteq V$ tests whether two distributions $\Pr[v_j \mid v_i,  Z, \doo(S)]$ and $\Pr[v_j \mid Z, \doo(S \cup \{v_i\})]$ are identical or not.
\end{assumption}

The intuition of the do-see test is as follows: if $v_i$ and $v_j$ are the only two nodes in the graph $G$ with $v_i \rightarrow v_j$, then, $\Pr[v_j \mid v_i] = \Pr[v_j \mid \doo(v_i)]$ \text{ iff there exists no latent that effects both of them}. This follows from the {\em conditional invariance} principle~\citep{bareinboim2012local} (or page 24, property 2 in~\citep{pearl}). Therefore, the presence or absence of latents can be established by invoking a DT-oracle.

As we seek to minimize the number of interventions, our goal is to create intervention sets that contain $p$-colliders between every pair of variables that share an edge in $G$. However, in Lemmas~\ref{lem:first},~\ref{lem:second} we argue that it is not sufficient to consider interventions with only $p$-colliders. We must also intervene on $Pa(v_i)$ to detect a latent between $v_i \rightarrow v_j$. The main idea behind~\LatentsWEdges~is captured by the following two lemmas. 

\begin{lemma} [No Latent Case] \label{lem:first}
Suppose $v_i \rightarrow v_j \in G$ and $v_i, v_j \not \in B$, and $P_{ij} \subseteq B$ then, $\Pr[v_j \mid v_i, \Pa(v_j), \doo(\Pa(v_i) \cup B)] = \Pr[v_j \mid  \Pa(v_j), \doo(\{v_i\} \cup \Pa(v_i) \cup B)]$ if there is no latent $l_{ij}$ with $ v_i \leftarrow l_{ij} \rightarrow v_j$.
\end{lemma}
\begin{lemma} [Latent Case] \label{lem:second}
Suppose $v_i \rightarrow v_j \in G$ and $v_i, v_j \not \in B$, and $P_{ij} \subseteq B$, then, $\Pr[v_j \mid v_i, \Pa(v_j), \doo(\Pa(v_i) \cup B)] \neq \Pr[v_j \mid  \Pa(v_j), \doo(\{v_i\} \cup \Pa(v_i) \cup B)]$ if there is a latent $l_{ij}$ with $ v_i \leftarrow l_{ij} \rightarrow v_j$.
\end{lemma}

From Lemmas~\ref{lem:first},~\ref{lem:second}, we know that to identify a latent $l_{ij}$ between $v_i \rightarrow v_j$, we must intervene on all the $p$-colliders between them with $Pa(v_i) \cup \{ v_i \}$. To do this, we again construct random intervention sets. Let $B_t \subseteq V$ for $t \in \{1, 2, \cdots, 72\tau'\log n\}$ be constructed by including  every variable $v_i \in V$ with probability $1-{1}/{\tau'}$ where $\tau'=\max\{\tau,2\}$. Let $\BBB_\tau = \{B_1,\cdots,B_{72\tau'\log n}\}$ be the collection of the sets. Consider a pair $v_i \rightarrow v_j$.  To obtain the interventions given by the above lemmas, we iterate over all sets in $\mathcal{B_\tau}$ and identify all the  sets containing $v_i$, but not $v_j$. From these sets, we remove $v_i$ to obtain $\mathcal{B}_{ij}$. These new interventions are then used in~\LatentsWEdges~to perform the required distribution tests using a DT-oracle on the interventions $B \cup Pa(v_i)$ and $B \cup Pa(v_i) \cup \{ v_i \}$ for every $B \in \mathcal{B}_{ij}$. We can show:

\begin{proposition} \label{prop:2}
Let $\GGG(V \cup L, E \cup E_L)$ be a $\tau$-causal graph with observable graph $G(V,E)$.\\ \LatentsWEdges~with $O(n\tau \log n + n \log n)$ interventions recovers all latents effecting pairs of adjacent nodes in the observable graph $G$ with probability at least $1-{1}/{n^2}$.
\end{proposition}

\pparagraph{Putting it all Together.} Using Propositions~\ref{prop:obs},~\ref{prop:1}, and~\ref{prop:2}, we get the following result.
Note that $\tau \leq n-2$.
\begin{theorem}
\label{thm:lat}
Given access to a $\tau$-causal graph $\GGG= \GGG(V \cup \lat, E  \cup E_L)$ through Conditional Independence (CI) and Distribution Testing (DT) oracles, Algorithms~\RecoverG,~\LatentsNEdges, and~\LatentsWEdges put together recovers $\GGG$ with $O(n\tau \log n + n \log n)$ interventions, with probability at least $1-O(1/n^2)$ $($where $|V|=n)$.
\end{theorem}


\section{Experiments}\label{sec:experiments}
In this section, we compare the total number of interventions required to recover causal graph $\GGG$ parameterized by $p$-colliders (See Section~\ref{sec:unweighted}) vs.\ maximum degree utilized by~\citep{neurips17}. 

{Since the parameterization of these two results are different, a direct comparison between them is not always possible. If $\tau=o(d^2/n)$, we use fewer interventions than ~\citet{neurips17} for recovering the causal graph. Roughly, for any $0\leq \eps \leq 1$,  (a) when $\tau < n^\eps, d > n^{(1+\eps)/2}$, our bound is better, 
(b) when $\tau > n^\eps, \tau < d < n^{(1+\eps)/2}$, then we can identify latents using the algorithms of~\citet{neurips17}  after using our algorithm for observable graph recovery, and 
(c) when $\tau > d > n^\eps, d < n^{(1+\eps)/2}$, the bound in~\citet{neurips17} is better.
}

In this section, our main motivation is to show that $p$-colliders can be a useful measure of complexity of a graph.  As discussed in Section~\ref{sec:intro}, even few nodes of high degree could make $d^2$ quite large. 

\noindent\textbf{Setup}.
We demonstrate our results by considering sparse random graphs generated from the families of:  (i) Erd\"os-R\'enyi random graphs $G(n,c/n)$ for constant $c$, (ii) Random bipartite graphs generated using $G(n_1,n_2,c/n)$ model, with partitions $L$, $R$ and edges directed from $L$ to $R$, (iii) Random directed trees with degrees of nodes generated from power law distribution. In each of the graphs that we consider, we include latent variables by sampling $5\%$ of $\binom{n}{2}$ pairs and adding a latent between them. 

\noindent\textbf{Finding $p$-colliders.}
Let $\GGG$ contain observable variables and the latents. To find $p$-colliders between every pair of observable nodes of $\GGG$, we enumerate all paths between them and check if any of the observable nodes on a path can be a possible $p$-collider. As this became practically infeasible for larger values of $n$, we devised an algorithm that runs in polynomial time (in the size of the graph) by constructing an appropriate flow network and finding maximum flow in this network. Please refer to Appendix~\ref{app:experiments} for more details. 

\begin{figure}[h]
\centering
\includegraphics[width=2in]{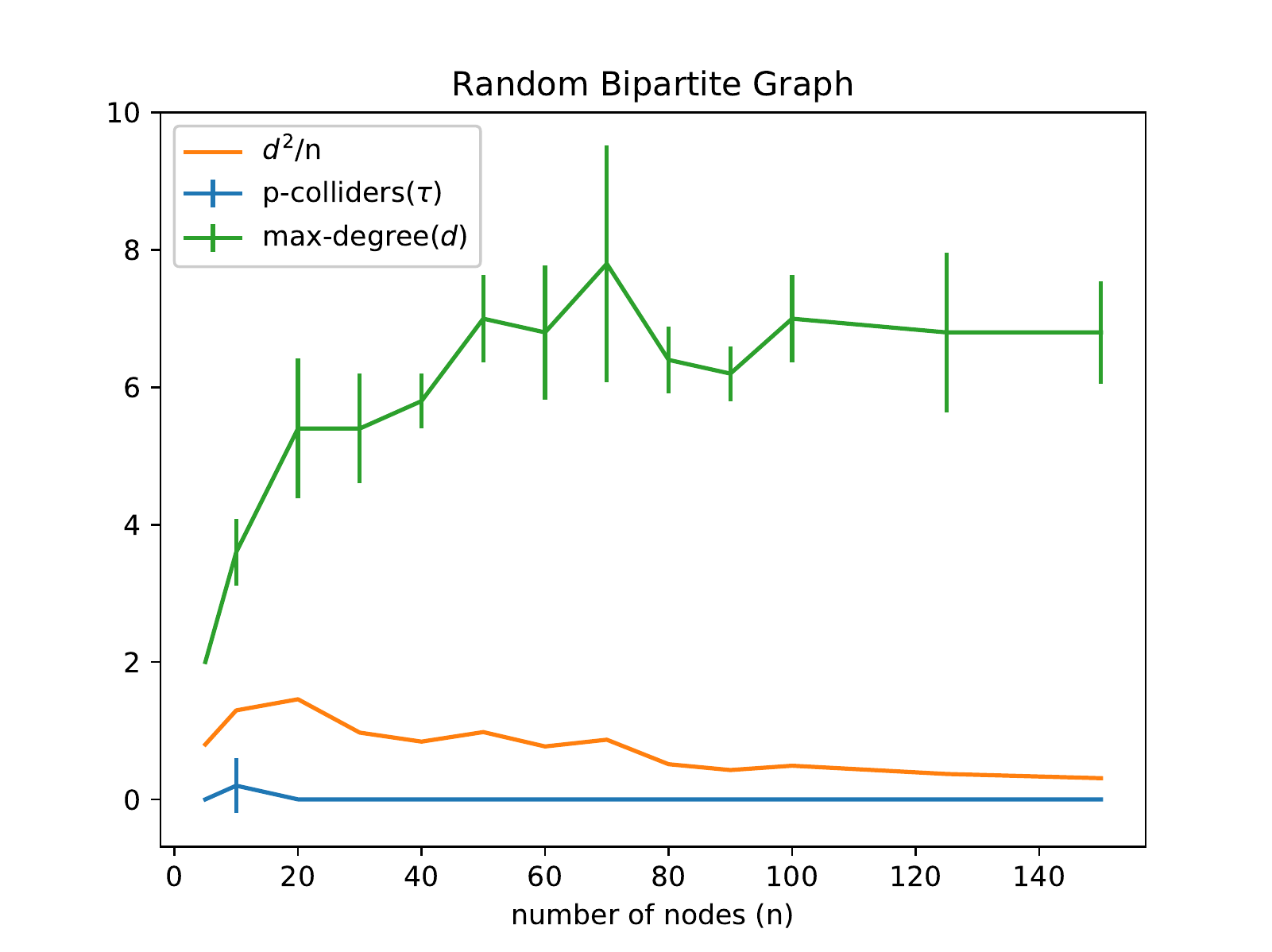}
\caption{Comparison of $\tau$ vs.\ maximum degree in sparse random bi-partite graphs.
}
\label{fig:plot_main}
\end{figure}

\noindent\textbf{Results.}
In our plots (Figure~\ref{fig:plot_main}), we compare the maximum undirected degree $(d)$ with the maximum number of $p$-colliders between any pair of nodes (which defines $\tau$).
We ran each experiment $10$ times and plot the mean value along with one standard deviation error bars. 

For random bipartite graphs, that can be used to model causal relations over time, we use equal partition sizes $n_1 = n_2 = n/2$ and plot the results for $\GGG(n/2,n/2, c/n)$ for constant $c = 5$. We observe that the behaviour is uniform for small constant values of $c$. In Figure~\ref{fig:plot_main}, we observe that the maximum number of $p$-colliders($\tau$) is close to zero for all values of $n$ while the values of $d^2/n$ using the mean value of $d$, is significantly higher. So, in the range considered our algorithms use fewer interventions. We show similar results for other random graphs in Appendix~\ref{app:experiments}.

Therefore, we believe that minimizing the number of interventions based on the notion of $p$-colliders is a reasonable direction to consider.

\section{Concluding Remarks}

We have studied how to recover a causal graph in presence of latents while minimizing the intervention cost. In the linear cost setting, we give a $2$-approximation algorithm for ancestral graph recovery. This approximation factor can be improved to $(1+\epsilon)$ under some additional assumptions. Removing these assumptions would be an interesting direction for future work. In the identity cost setting, we give a randomized algorithm to recover the full causal graph, through a novel characterization based on $p$-colliders. In this setting, understanding the optimal intervention cost is open, and an important direction for research.

While we focus on non-adaptive settings, where all the interventions are constructed at once in the beginning, an adaptive (sequential) setting has received recent attention~\citep{he2008active,shanmugam2015learning}, and is an interesting  direction in both our cost models.

\subsection*{Acknowledgements}
The first two named authors would like to thank Nina Mishra, Yonatan Naamad, MohammadTaghi Hajiaghayi and Dominik Janzing for many helpful discussions during the initial stages of this project.
This work was partially supported by NSF grants CCF-1934846, CCF-1908849, and   CCF-1637536.
\newpage
\bibliographystyle{plainnat}
\bibliography{references}

\newpage
\appendix
\noindent\rule{\textwidth}{1pt}
\begin{center}
{\Large \sf Appendix }
\end{center}
\noindent\rule{\textwidth}{1pt}


\section{Missing Details from Section~\ref{sec:lincost}} \label{app:lincost}
\begin{lemma} [Lemma~\ref{lem:lb-strong} Restated]
Suppose $\mathcal{S} = \{ S_1, S_2, \cdots, S_m \}$ is a collection of subsets of $V$. For a given causal graph $G$ if $\Anc(G)$ is recovered using CI-tests by intervening on the sets $S_i \in \mathcal{S}$. Then, $\mathcal{S}$ is a strongly separating set system.
\end{lemma}
\begin{proof}
Suppose $\mathcal{S}$ is not a strongly separating set system. If there exists a pair of nodes $(v_i, v_j)$ such that every set $S_k \in \mathcal{S}$ contains none of them, then, we cannot recover the edge between these two nodes as we are not intervening on either $v_i$ or $v_j$ and the results of an independence test $v_i \indep v_j$ might not be correct due to the presence of a latent variable $l_{ij}$ between them. Now, consider the case when only one of them is present in the set system. Let $(v_i, v_j)$ be such that $\forall S_k :  S_k \cap \{ v_i, v_j \}  = \{ v_i \} \Rightarrow v_i \in S_k, v_j \not \in S_k $. We choose our graph $G_{ij}$ to have two components $\{v_i, v_j\}$ and $V \setminus \{v_i, v_j\}$; and include the edge $v_j \rightarrow v_i$ in it. Our algorithm will conclude  from CI-test $v_i \indep v_j \mid \doo(S_k)$ that $v_i$ and $v_j$ are independent. However, it is possible that $v_i \notindep v_j$ because of a latent $l_{ij}$ between $v_i$ and $v_j$, but $v_i \indep v_j \mid \doo(S_k)$ as intervening on $v_i$ disconnects the $l_{ij} \rightarrow v_i$ edge. Therefore, our algorithm cannot distinguish the two cases $v_j \rightarrow v_i$ and $v_i \leftarrow l_{ij} \rightarrow v_j$ without intervening on $v_j$. For every $\mathcal{S}$ that is not a strongly separating set system, we can provide a $G_{ij}$ such that by intervening on sets in $\mathcal{S}$, we cannot recover $\Anc(G_{ij})$ correctly. 
\end{proof}
\subsection{Missing Details from Section~\ref{sec:2approx}} \label{app:2approx}

We first argue that the matrix returned by Algorithm~\SSMatrix~is indeed a strongly separating matrix.

\begin{lemma} \label{lem:ss_correctness}
The matrix $U$ returned by Algorithm~\SSMatrix~is a strongly separating matrix.  
\end{lemma}
\begin{proof}
Consider any two nodes $v_i$, $v_j$ with corresponding row vectors $U(i)$ and $U(j)$. Suppose $\|U(i)\|_1 = \|U(j)\|_1$.

By construction, $U(i) \neq U(j)$ they will differ in at least one coordinate. However, they have equal weights, so, there must exist one more coordinate such that the strongly separating condition holds. If $U(i)$ and $U(j)$ have weights $r^i \neq r^j$, then, $U(i, m'-\log n+r^i) = U(j, m'-\log n+r^j) = 1$ and $U(i, m'-\log n+r^j) = U(j, m'-\log n+r^i) = 0$ by construction outlined in the Algorithm~\SSMatrix. This proves that the matrix $U$ returned is a strongly separating matrix.
\end{proof}
The following inequalities about Algorithm~\SSMatrix~will be useful in analyzing its performance.
\begin{lemma}\label{lem:ss_covering}.
For $m \geq 66 \log n$ and $m'$ as defined in Algorithm~\SSMatrix, we have the following
\begin{itemize}
\item[(a)] $\sum_{t = 1}^{\log n} \binom{m'-\log n}{t} \geq n$ (i.e., there are enough vectors of weight $\le \log n$ only using $m'-\log n$ columns to assign a unique vector to each variable).
\item[(b)] Let $i^*$ be the smallest integer s.t.\ $\sum_{t = 1}^{i^*} {m' \choose t} \geq n$. Then, $\sum_{t = 1}^{2i - 1} {m' - \log n \choose t} \geq \sum_{t = 1}^{i} {m' \choose t}$ for all $i \in \{ 2, \ldots, i^* \}$.
\end{itemize}
\end{lemma}
\begin{proof}
\begingroup
\allowdisplaybreaks
Let $m \geq { 66} \log n$. From Algorithm~\SSMatrix, we have $m' = m - a_1$ for all guesses $1 \leq a_1 \leq \frac{2m}{3}$. By reserving the last $``\log n"$ columns in Algorithm~\SSMatrix, we want to make sure that $m'-\log n$ can fully cover $n$ nodes with weight at most $\log n$. We have :
\begin{align*}
    m' = m - a_1 \geq \frac{m}{3} &\geq 22\log n \mbox{ and }\\
    \sum_{t = 1}^{\log n} {m'-\log n \choose t} \geq {m' - \log n \choose \log n}
    &\geq \left(\frac{21\log n}{\log n} \right)^{\log n} > n.
\end{align*}
Moving onto Part (b). Let $i^*$ be the minimum value of $i$ such that $\sum_{t = 1}^{i^*} {m' \choose t} \geq n$. Consider $i$ such that $2 \leq i \leq i^*$: 
\begin{align*}
\sum_{t = 1}^{2i - 1} {m' - \log n \choose t}  &\geq {m' - \log n \choose 2i-1 }. 
\end{align*}
\text{Consider now the right hand side: }
\begin{align*}
\sum_{t = 1}^{i} {m' \choose t} &\leq  \sum_{t = 0}^{i} {m' \choose t}  \leq \sum_{t=0}^i \frac{m'^t}{t!} \leq \sum_{t=0}^i \frac{i^t}{t!} \left( \frac{m'}{i} \right)^t  \leq \e^i \left( \frac{ m'}{i} \right)^{i}.   
\end{align*}
We inductively show that $\text{ for all } i \geq 2$
\[ 
    \frac{\left( \frac{em'}{i} \right)^i}{\binom{m'-\log n}{2i-1} } \leq 1.
\]
Let $i = 2$. For $m \geq \frac{m'}{3} \geq 50\left( \frac{22}{21} \right)^3$ we have
\begin{align*}
    \frac{(em'/2)^2}{\binom{m'-\log n}{3}} &\leq \frac{(em'/2)^2}{(\frac{m'-\log n}{3})^3} \leq 50\left( \frac{22}{21} \right)^3\frac{m'^2}{m'^3} \leq 1. \\
\end{align*}
Assume the inequality is correct for some $i > 2$. Now, we show that it must also hold for $i+1$.
\begin{align*}
    \frac{\left( \frac{em'}{i+1} \right)^{i+1}}{\binom{m'-\log n}{2i+1} } &=      \frac{\left( \frac{em'}{i+1} \right)^i \frac{em'}{i+1}\binom{m'-\log n}{2i-1}}{\binom{m'-\log n}{2i-1}\binom{m'-\log n}{2i+1} } 
\leq  \frac{\left( \frac{em'}{i} \right)^i \frac{em'}{i+1}\binom{m'-\log n}{2i-1}}{\binom{m'-\log n}{2i-1}\binom{m'-\log n}{2i+1} } 
    \leq \frac{ \frac{em'}{i+1}\binom{m'-\log n}{2i-1}}{\binom{m'-\log n}{2i+1} }.
\end{align*}
For ease of notation, denote $a = m'-\log n \geq m'(1-\frac{1}{22}) \geq 21\log n$.

Consider the binary entropy function $H(x)= -x
\log x - (1-x)\log (1-x)$. For $x \in [\frac{2i-1}{a}, \frac{2i+1}{a} ]$, $H(x)$ is an increasing function. For some value of $x$ in the range we have : 
\[ \frac{H(\frac{2i+1}{a}) - H(\frac{2i-1}{a})}{\frac{2i+1}{a} - \frac{2i-1}{a}} = H'(x) = \log \left( \frac{1}{x} - 1 \right) \geq \log \left( \frac{a}{2i-1} - 1\right)\]
\[ \Longrightarrow  H\left(\frac{2i+1}{a}\right) - H\left(\frac{2i-1}{a}\right) \geq \frac{2}{a} \log \left( \frac{a}{2i-1} - 1\right).\]

Now, consider the fraction $${ \binom{a}{2i-1}}/{\binom{a}{2i+1} }.$$
Using the bound from (\cite{macwilliams1977theory}, Page 309)
\[
\sqrt{\frac{a}{8b(a-b)}}2^{mH(b/a)} \leq \binom{a}{b}\leq \sqrt{\frac{a}{2\pi b(a-b)}}2^{mH(b/a)},
\]
\begin{align*}
    { \binom{a}{2i-1}}/{\binom{a}{2i+1} } &\leq {\sqrt{8(2i+1)(a-2i-1)/2\pi(2i-1)(a-2i+1)}}/{2^{a H(\frac{2i+1}{a}) - H(\frac{2i-1}{a})}}\\
    &\leq {\sqrt{20/3\pi}}/{2^{a H(\frac{2i+1}{a}) - H(\frac{2i-1}{a})}}\\
    &\leq {\sqrt{20/3\pi}}/{2^{2 \log \left( \frac{a}{2i-1} - 1\right)}} \\
    &= \frac{ \sqrt{20/3\pi} }{\left( \frac{a}{2i-1} - 1\right)^2}.
\end{align*}
Combining the above, we have : 
\begin{align*}
\frac{\left( \frac{em'}{i+1} \right)^{i+1}}{\binom{m'-\log n}{2i+1} } &\leq \frac{ \frac{m'}{i+1} \sqrt{20e^2/3\pi}}{\left( \frac{a}{2i-1} - 1\right)^2}\\
&\leq \frac{ 4 m' i \sqrt{20e^2/3\pi}}{\left( a-2i \right)^2}\\
&\leq \frac{ 4m'\log n   \sqrt{20e^2/3\pi}}{m'^2\left( 1-3/22 \right)^2} = \frac{ 4\log n   \sqrt{20e^2/3\pi}}{m'\left( 1-3/22 \right)^2} \leq \frac{21.2 \log n}{m'} \leq 1.
\end{align*}

Therefore, we have for all $i \geq 2$
\[ 
    {\left( \frac{em'}{i} \right)^i}/{\binom{m'-\log n}{2i-1} } \leq 1 
\]
\[ \Longrightarrow
\sum_{t = 1}^{i} {m' \choose t} \leq \left( \frac{em'}{i} \right)^i \leq \binom{m'-\log n}{2i-1} \leq \sum_{t = 1}^{2i - 1} {m' - \log n \choose t}.
\]
\endgroup
\end{proof}

Let $c_U = \sum_{j=1}^n c(v_j) \| U(j)\|_1$ be value of objective for the matrix $U$ returned by Algorithm~\SSMatrix.

Consider $U_{\OPT}$, and let $V^{(1)}_{\OPT}$ represent all nodes that are assigned weight $1$ in it (nodes which have only one $1$ in their row). Let $c^{(1)}_{\OPT}$ denote the sum of  cost of the nodes in $V^{(1)}_{\OPT}$. In our Algorithm~\SSMatrix, we maintain a guess for the size of $V^{(1)}_{\OPT}$ as $a_1$. We want to guess the exact value of $|V^{(1)}_{OPT}| \leq m$. However, we only guess $a_1$ until $\frac{2m}{3}$, so that the remaining columns can be used to obtain a valid separating matrix (for each of our guesses) as observed in Lemma~\ref{lem:ss_covering}. We show that the cost contribution of nodes in $V^{(1)}_{\OPT}$ (by allowing this slack in our guesses) due to Algorithm~\SSMatrix is not far away from $c^{(1)}_{\OPT}$. 

First, we show that for any weight $i \geq 2$ node in $U_{\OPT}$, the output $U$ of Algorithm~\SSMatrix~assigns vectors with weight at most $2i$ and for a weight $1$ node, we show that the weight assigned by $U$ is at most $3$.

\begin{lemma} \label{lem:wt_distr_2approx}
Algorithm~\SSMatrix~assigns a weight of
\begin{CompactItemize}
    \item[(a)]  at most $3$ for a weight $1$ node in $U_{\OPT}$.
    \item[(b)]  at most $2i$ for a node of weight $i$ in $U_{\OPT}$ for $i \geq 2$.
\end{CompactItemize} 
\end{lemma} 
\begin{proof}
$(a)$ Let $V$ denote sorted (in the decreasing order of cost) order of nodes. Suppose we assign unique length-$m$ vectors starting from weight $1$ to the nodes in the order $V$. Let the assignment of vectors be denoted by $\tilde{U}$. It is easy to observe that this described assignment $\tilde{U}$ is not a strongly separating matrix. However, any strongly separating matrix $U$ is such that the vector assigned to any node $v_i$ in $U$ has weight at least that in $\tilde{U}$ i.e., $\| U(i) \|_1 \geq \| \tilde{U}(i) \|_1$. As $U$ can be any strongly separating matrix, it also holds for $U_{\OPT}$ giving us $\| U_{\OPT}(i) \|_1 \geq \| \tilde{U}(i) \|_1$. \\ \\
  The number of weight $1$ nodes possible in the assignment $\tilde{U}$ is ${m \choose 1}$ and therefore, $|V^{(1)}_{\OPT}| \leq m$. Consider all the nodes of weight $\leq 3$ in $U$ assigned by Algorithm~\SSMatrix. After discarding the first $m' = m - a_1$ columns assuming our guess $a_1$ in the current iteration, $U$ starts assigning vectors with weight $1$ in the remaining $m'-\log n$ while setting a `row weight indicator bit' in the last $\log n$ columns. In order to obtain nodes of weight $\leq 3$, in $U$, we include vectors of weight $\leq 2$ in the $m'-\log n$ columns. Therefore, total number of such nodes is $a_1 + {m' - \log n \choose 1} + {m'-\log n \choose 2} $.
\begin{align*}
a_1 + {m' - \log n \choose 1} + {m'-\log n \choose 2} &\geq {m' - \log n \choose 1} + {m'-\log n \choose 2}\\
&\geq {m' - \log n}+ \left( \frac{m'-\log n}{2}\right)^2 && \text{(using${m' \choose k} \geq \left(\frac{m'}{k}\right)^k$)}\\
&\geq m  \geq |V^{(1)}_{\OPT}|. && \text{(using $m' \geq \frac{m}{3}$ and $m \geq 66 \log n$)}
\end{align*}
Therefore, every weight $1$ node in $U_{\OPT}$ is covered by a vector in $U$ with weight $\leq 3$.

$(b)$ First we argue that using an appropriate $m'$, we can give a construction of $\tilde{U} \in \{0, 1\}^{n\times m'}$ (similar to case (a)) such that weight of node $v_j$ in $\tilde{U}$ is at most the weight in $U_{\OPT}$ for all nodes of weight more than $2$ in $U_{\OPT}$. Let $m' = m - \frac{2m}{3}$. In other words, we are considering the guess $a_1 = \frac{2m}{3}$. As our algorithm $U$ considers all the guesses and returns $U$ with the lowest cost, arguing that our lemma holds for this guess is sufficient. For this value of $m'$, let $\tilde{U}$ be constructed using vectors from $\{ 0, 1\}^{m'}$ in the increasing order of weight, starting with weight $1$.

When $m' = \frac{m}{3}$, it is possible that a node in $U_{\OPT}$ can be assigned a vector of weight $1$ from $\{ 0, 1 \}^{m'}$ (this can happen when $|V^{(1)}_{\OPT}| \geq \frac{2m}{3}$). As $\tilde{U}$ assigns weights in the increasing order occupying the entire $m'$ columns, it will not result in a strongly separating matrix. Therefore, any node $v_j$ with weight $i \geq 2$ in $U_{\OPT}$, will be assigned a weight of at most $i$ in $\tilde{U}$.

We know that the number of vectors of weight at most $i$ in $\tilde{U}$ is equal to $\sum_{t = 1}^{i} {m' \choose t}$ and number of vectors with weight at most $2i-1$ using $m'-\log n$ columns of $U$ is equal to $ \sum_{t = 1}^{2i - 1} {m' - \log n \choose t} $. As Lemma~\ref{lem:ss_covering} holds for all guesses of $a_1$, we have $\sum_{t = 1}^{i} {m' \choose t} \leq \sum_{t = 1}^{2i - 1} {m' - \log n \choose t} $ for all $i \geq 2$. Using induction, we can observe that $v_j$ is assigned a vector in $\{ 0, 1\}^{m'-\log n}$ with weight at most $2i - 1$. As $U$ obtained from Algorithm~\SSMatrix~mimics the construction used in $\tilde{U}$ over $m' - \log n$ columns, we have that weight of node $v_j$ in $U$ using $m'- \log n$ columns is at most $2i-1$. Combining it with the `row weight indicator' bit we set to $1$ in the last $\log n$ columns gives us the lemma.
\end{proof}

 In our next lemma shows that the sum of contribution of the nodes in $V^{(1)}_{\OPT}$ to $c_U$ is at most twice that of $c^{(1)}_{\OPT}$.
Combining this with Lemma~\ref{lem:wt_distr_2approx}, we show that $U$ achieves a $2$-approximation.

\begin{lemma} \label{lem:2approx_wt_1}
Let $c^{(1)}_U = \sum_{v_i \in V^{(1)}_{\OPT}} c(v_i) \|U(i)\|_1$
for the matrix $U$ returned by Algorithm~\SSMatrix, then $c^{(1)}_U \leq 2 c^{(1)}_{OPT}$.
\end{lemma}
\begin{proof}
\begingroup
\allowdisplaybreaks
Suppose $a_1$ represents our guess for the number of weight $1$ vectors and $a^*_1$ represent the number of weight $1$ vectors in $U_{\OPT}$ i.e, $|V^{(1)}_{\OPT}| = a^*_1$. In Algorithm~\SSMatrix, we use the following bounds for our guess $0 \leq a_1 \leq {2m}/{3}$.  If $a^*_1 \leq \frac{2m}{3}$, then it would have been one of our guesses. As we take minimum among all the guesses, we have $c^{(1)}_U = c^{(1)}_{OPT}$ in such a case.\\ \\
Consider the case when $a^*_1 > \frac{2m}{3}$. Let $V^{(1)}_{\OPT} = \{v_1, v_2, \cdots v_{a^*_1} \}$ represent an ordering of nodes in the decreasing ordering of cost that are assigned weight 1 in $U_{\OPT}$. Consider the contribution of only weight $1$ nodes to  $c_{\OPT}$. 
We have $$c^{(1)}_{\OPT} = \sum_{k=1}^{a^*_1} c(v_k) \geq \sum_{k=1}^{2m/3} \frac{2m}{3} c(v_k) \geq \frac{2m}{3}c(v_{2m/3}).$$
We will look at the case when our guess $a_1 $ reaches $a_1 = \frac{2m}{3}$ and argue about the cost for this particular value of $a_1$. As we are taking minimum over all the guesses, we are only going to do better and our approximation ratio will only be better. Among the nodes $\{v_1, v_2, \cdots v_{a^*_1} \}$ first $\frac{2m}{3}$ nodes would be assigned weight $1$ by $U$. From Lemma~\ref{lem:wt_distr_2approx}, we have that for the remaining $a^*_1 - \frac{2m}{3}$ nodes, Algorithm~\SSMatrix~might assign a weight $2$ or weight $3$ vector in $U$. 
\begin{align*}
c^{(1)}_U &\leq \sum_{i=1}^{2m/3} c(v_i) + 3 \sum_{j=2m/3+1}^{a^*_1} c(v_j) = \sum_{i=1}^{a^*_1} c(v_i) + 2 \sum_{j=2m/3+1}^{a^*_1} c(v_j)\\
    &\leq \sum_{i=1}^{a^*_1} c(v_i) + 2 \left ({a^*_1} - \frac{2m}{3}\right ) c(v_{2m/3+1}) \\
    & \leq c^{(1)}_{OPT} + 2 \left ({a^*_1} - \frac{2m}{3} \right) c(v_{2m/3}) && \text{(since $c(v_{2m/3}) \geq c(v_{2m/3+1})$ )}\\
&\leq c^{(1)}_{OPT} + \frac{2m}{3} c(v_{2m/3}) && \text{ (since ${a^*_1} \leq m$ ) }\\
&\leq 2 c^{(1)}_{OPT}.
\end{align*}
This completes the proof of the lemma.
\endgroup
\end{proof}

\begin{theorem} [Theorem~\ref{thm:ss_2approx} Restated]
Let $m \geq 66\log n$ and $U$ be the strongly separating matrix returned by Algorithm~\SSMatrix. Let $c_U = \sum_{j=1}^n c(v_j)\, \| U(j) \|_1$. Then, $$c_U \leq 2 \cdot c_{\OPT},$$ where $c_{\OPT}$ is the objective value associated with optimum set of interventions corresponding to $U_{\OPT}$.
\end{theorem}
\begin{proof}
From Lemma~\ref{lem:ss_correctness}, we know that matrix returned by Algorithm~\SSMatrix~given by $U$ with cost $c_U$ is a strongly separating matrix. 
Consider a strongly separating matrix $U_{\OPT}$ that achieves optimum objective value $c_{\OPT}$. Let $V^{(1)}_{\OPT}$ represent all nodes that are assigned weight $1$ in $U_{\OPT}$. Let $c^{(1)}_U$ denote the cost of nodes in $V^{(1)}_{\OPT}$ using $U$ returned by Algorithm~\SSMatrix~and $c^{(1)}_{\OPT}$ represents that of $U_{\OPT}$. We have $c_{\OPT} = c^{(1)}_{OPT} + \sum_{j : \| U_{\OPT}(j) \|_1 \geq 2} c(v_j)\, \| U_{\OPT(j)} \|_1$.
\begin{align*}
    c_U &= c^{(1)}_U + \sum_{j : \| U_{\OPT}(j) \|_1 \geq 2} c(v_j)\, \| U(j) \|_1 \\
        &\leq c^{(1)}_U + \sum_{j : \| U_{\OPT}(j) \|_1 \geq 2} c(v_j)\, 2\| U_{\OPT}(j) \|_1 && \text{(from  Lemma~\ref{lem:wt_distr_2approx})}\\
        &\leq  2 c^{(1)}_{OPT} + 2\sum_{j : \| U_{\OPT}(j) \|_1 \geq 2} c(v_j)\, \| U_{\OPT}(j) \|_1 && \text{(from  Lemma~\ref{lem:2approx_wt_1})}\\
        &\leq 2c_{\OPT}.
\end{align*}
This completes the proof of the theorem.
\end{proof}

\subsection{Missing Details from Section~\ref{sec:kruskal-katona}} \label{app:kruskal-katona}
In this section, we present our algorithm that achieves an improved $1+\eps$-approximation in the linear cost model setting under mild assumptions on the cost of the nodes and the number of interventions. The algorithm is adapted from that proposed by~\citet{hyttinen2013experiment} whose work drew connections between causality and known separating system constructions in combinatorics. In particular,~\citep{hyttinen2013experiment} considered a setting where given $n$ variables and $m$, the goal is to construct $k$ sets that are strongly separating with the objective of minimizing the average size of the intervention sets. Stated differently this provides an algorithm for solving~\ref{eqn:obj} when $c(v)=1$ for all nodes $v \in V$.

In Section~\ref{app:alg}, we adapt the algorithm from~\citep{hyttinen2013experiment} to deal with the case where each node could have a different cost value. Our main contribution is to show that this adaptation constructs a set of interventions which achieves an objective value in the linear cost model that is within a factor $1+\eps$ times of the optimum under some mild restrictions. In Section~\ref{app:prelimkruskal}, we start with some definitions and statements from the combinatorics that will prove useful for stating and analyzing the algorithm.
\subsection{Combinatorics Preliminaries}
\label{app:prelimkruskal}
\begin{definition}(antichain).
Consider a collection $\mathcal{S}$ of subsets of $\{v_1, v_2, \cdots v_n\}$ such that for any two sets $S_i, S_j \in \mathcal{S}$, we have $S_i \not\subset S_j$ and $S_j \not\subset S_i$. Then, such a collection $\mathcal{S}$ is called an antichain.
\end{definition}

We provide a lemma that shows that an antichain can also be represented as a strongly separating matrix.

\begin{lemma}
\label{lem:antichain_ss}
Let $\mathcal{T} = \{ T_1, T_2, \cdots T_n \}$ be an antichain defined on $\{ 1, 2, \cdots  m\}$. Construct a  matrix $U \in \{0,1\}^{n \times m}$ where $U(i,j) = 1$ iff $T_i$ contains $j$. Then $U$ is a strongly separating matrix.
\end{lemma}
\begin{proof}
From the definition of antichain, for any two sets $T_i, T_j \in \mathcal{T}$, there exists $k$ and ${k'}$ such that $k \in T_i \setminus T_j$ and  $k' \in T_j \setminus T_i$. So, we have
$U(i,k) = U(j, k') = 1$ and $U(i, k') = U(j, k) = 0$. It follows that $U$ is a strongly separating matrix from the definition.
\end{proof}
In the previous lemma, we gave a construction of a strongly separating matrix that corresponds to an antichain. In the next lemma, we show that given a strongly separating system, we can also obtain a corresponding antichain. 
\begin{lemma}
\label{lem:ss_antichain}
Let $\mathcal{S} = \{ S_1, S_2, \cdots S_m \}$ be a strongly separating set system defined on $\{v_1, v_2, \cdots v_n\}$. Construct a strongly separating matrix $U \in \{0,1\}^{n \times m}$ where $U(i,j) = 1$ iff $S_j$ contains $v_i$. Define a collection of sets $\mathcal{T} = \{ T_1, T_2, \cdots T_n \}$ defined over the column indices of $U$ i.e., $\{1, 2, \cdots m\}$ such that $j \in T_i$ iff $U(i, j) = 1$. Then, $\mathcal{T}$ is an antichain.
\end{lemma}
\begin{proof}
From the definition of strongly separating system, we have for every two nodes $v_i, v_j \in \mathcal{S}$, there exists $S_k$ and $S_{k'}$ such that $v_i \in S_k \setminus S_{k'}$ and  $v_j \in S_{k'} \setminus S_k$. This implies $k \in T_i \setminus T_j$ and $k' \in T_j \setminus T_i$ as
$U(i,k) = U(j, k') = 1$ and $U(i, k') = U(j, k) = 0$. Therefore, for every two sets $T_i$ and $T_j$ in $\mathcal{T}$, we have $T_i \not\subset T_j$ and $T_j \not\subset T_i$. Hence, $\mathcal{T}$ is an antichain.
\end{proof}

\begin{lemma} LYM inequality~\cite{jukna2011extremal}. Suppose $\mathcal{S}$ represent an antichain defined over the elements $\{1, 2, \cdots m\}$. Let $a_k = |\{ T \mid T \in \mathcal{S} \text{ where }|T| = k\}|$ defined for all $k \in [m]$, then,
$$\sum_{k=0}^m \frac{a_k}{{m \choose k}} \leq 1.$$
\label{lem:LYM_ineq}
\end{lemma}

\begin{definition}\cite{jukna2011extremal}. A neighbor of a binary vector $v$ is a vector which can be obtained from $v$ by flipping one of its 1-entries to 0. A shadow of a set $A \subseteq \{0, 1\}^m$ of vectors is the set of all its neighbors and denoted by $\partial(A)$.
\end{definition}

Suppose $A \subseteq \{0, 1\}^m$ consists of weight $k$ vectors i.e., for all $v \in A, \norm{v}_1 = k$ . Then, there is an interesting representation for $|A|$ i.e., size of $A$ called the \textit{k-cascade} form, 
\[ |A| = \binom{a_k}{k} + \binom{a_{k-1}}{k-1} + \binom{a_{k-2}}{k-2} + \cdots + \binom{a_s}{s} \text{ where } a_k > a_{k-1} > \cdots a_s \geq s \geq 1.\]
Moreover, this representation is unique and for every $|A| \geq 1$, there exists a \textit{k-cascade} form. Given a set $A$ of such vectors, we can make the following observation. 

\begin{observation}\label{obs:shadow}
Let $B \subseteq \{ 0, 1 \}^m$ be a collection of vectors with weight exactly $k-1$. If $A \cup B$ is an antichain, then, $B \cap \partial(A) = \phi$.
\end{observation}

The above observation implies that if we want to maximize the number of weight $k-1$ vectors to get a collection of weight $k$ and $k-1$ vectors that form an antichain, then, we have to choose weight $k$ vectors that has a small shadow. Now, we describe the statement of the famous Kruskal-Katona theorem that gives a lower bound on the size of shadow of $A$. 
\begin{theorem} [Kruskal-Katona Theorem~\cite{jukna2011extremal}]
Consider a set $A \subseteq \{ 0, 1 \}^m$ of vectors such that for all $v\in A, \norm{v}_1 = k$ and the k-cascade form is 
\[ |A| = \binom{a_k}{k} + \binom{a_{k-1}}{k-1} + \binom{a_{k-2}}{k-2} + \cdots + \binom{a_s}{s}.\]
Then,
\[ |\partial(A)| \geq \binom{a_k}{k-1} + \binom{a_{k-1}}{k-2} + \binom{a_{k-2}}{k-2} + \cdots + \binom{a_s}{s-1}.\]
\end{theorem}

\begin{definition}(Colexicographic Ordering)
Let $u$ and $v$ be two distinct vectors from $\{0, 1\}^m$. In the colexicographic ordering $u$ appears before $v$ if for some $i$, $u(i)= 0, v(i) = 1$ and $u(j) = v(j)$ for all $j >i$.
\end{definition}
We now state a result that the colexicographic ordering (or colex order) of all vectors of $\{0,1\}^m$ achieves the Kruskal-Katona theorem lower bound. Therefore, we can generate a sequence of any number of vectors with weight $k$ that has the smallest possible shadow.
\begin{lemma} Proposition 10.17 from \cite{jukna2011extremal}. 
\label{lem:shadow_size}
Using the first $T$ of weight $k$ vectors in the colex ordering of $\{ 0, 1 \}^m$, we can obtain a collection $A \subseteq \{0,1\}^m$  such that
$|\partial(A)| = \binom{a_k}{k-1} + \binom{a_{k-1}}{k-2} + \binom{a_{k-2}}{k-2} + \cdots + \binom{a_s}{s-1}$ where the $k$-cascade form of $ T = |A| = \binom{a_k}{k} + \binom{a_{k-1}}{k-1} + \binom{a_{k-2}}{k-2} + \cdots + \binom{a_s}{s}$.
\end{lemma}
\noindent We state the Flat Antichain theorem, that we will use later. 
\begin{theorem}[Flat Antichain Theorem] \cite{kisvolcsey2006flattening}
\label{thm:flat_antichain}
If $\mathcal{A}$ is an antichain, then, there exists another antichain $\mathcal{B}$ defined over same elements, such that $|\mathcal{A}| = |\mathcal{B}|$, $\sum_{A \in \mathcal{A}} |A| = \sum_{B \in \mathcal{B}} |B|$ and for every $B \in \mathcal{B}$, we have $|B| \in \{ d-1, d \}$ for some positive integer $d$.
\end{theorem}

\subsection{$(1+\eps)$-approximation Algorithm} \label{app:alg}
Algorithm~\epsSSMatrix is an adaptation of Algorithm 4 of~\citep{hyttinen2013experiment} for the linear cost model setting. From Lemma~\ref{lem:ss_antichain} and ~\ref{lem:antichain_ss}, it is clear that constructing a strongly separating set system is equivalent to constructing an antichain. A consequence of Flat Antichain theorem~\citep{kisvolcsey2006flattening} is that for every antichain $\mathcal{A}$ there is another antichain $\mathcal{B}$ of same size such that $\sum_{A \in \mathcal{A}} |A| = \sum_{B \in \mathcal{B}} |B| $ and $\mathcal{B}$ has sets of cardinality either $d$ or $d-1$ for some positive integer $d$. Therefore, the problem of finding a separating set system reduces to finding an appropriate antichain with weights $d$ and $d-1$ that minimizes the objective (assuming all nodes have cost equal to $1$).  
\begin{corollary}\cite{hyttinen2013experiment}
\label{cor:opt}
Flat Antichain theorem implies Algorithm~$4$ achieves optimal cost assuming all nodes have unit costs.
\end{corollary}
Algorithm~$4$ of~\citep{hyttinen2013experiment} is a consequence of Kruskal-Katona theorem; using colexicographic ordering we can maximize the $d-1$ weight vectors in an antichain of size $n$ consisting of weight $d$ and $d-1$ vectors. Therefore, choosing $d = k$ where $\binom{m}{k-1}<n\leq \binom{m}{k}$, they consider all possible number of vectors of weight $k$ and find the one with the minimum number of weight $k$ vectors. However, unlike~\citep{hyttinen2013experiment}, we have to deal with different costs of intervention for each node. We adopt a greedy strategy, where we assign the vectors (obtained using the previous algorithm) in the increasing order of weight to the nodes in the decreasing order of their costs. Observe that our Algorithm~\epsSSMatrix assigns vectors of weight $k-1$ or $k$ that are relatively high to the nodes with large costs. Surprisingly, we show that when the costs are bounded by $\approx \eps n^\eps$, and number of interventions $m \leq  n^\eps$, it achieves a $1+\epsilon$-approximation.

\begin{algorithm}
\begin{small}
\caption{\epsSSMatrix$(V,m)$}
\label{alg:uniform_wt}
\begin{algorithmic}[1]
     \State  Let $\tilde{U} \in \{0,1\}^{n \times m}$ be initialized with all zeros
     
     \State  Find $k$ satisfying $\binom{m}{k-1}<n\leq \binom{m}{k}$
	\For{$t = 0 \text{ to }n$}
	     \State  Let $A_t$ denote the first $t$ vectors in the colex ordering of $\{0,1\}^m$ with weight $k$. Calculate $|\partial(A_t)|$ using Lemma~\ref{lem:shadow_size}.
		\If{$t - |\partial(A_t)| + {m \choose k-1} \geq n $}
		     \State  For the rows $\tilde{U}(j)$ with $n-t+1 \leq j \leq n$ assign the vectors of weight $k$ using $A_t$
		     \State  For the rows $\tilde{U}(j)$ with $j \leq n-t$, assign vectors of weight $k-1$ from $\{0,1\}^{m}$ that are not contained in $\partial(A_t)$
		     \State  \textbf{break};
		\EndIf
	\EndFor
	\State Let $\zeta$ denote the ordering of rows of $\tilde{U}$ in the increasing order of weight. 
	\State For every $i \in [n]$ assign $U(i) = \tilde{U}(\zeta(i))$ where $i^{th}$ row of $U$ corresponds to the node with $i^{th}$ largest cost. 
	\State  \textbf{Return} $U$
\end{algorithmic}
\end{small}
\end{algorithm}

\begin{lemma}
\label{lem:kk_ss}
Let $U$ represent the output of Algorithm~\epsSSMatrix. Then, $U$ is a strongly separating matrix.
\end{lemma}
\begin{proof}
From Observation~\ref{obs:shadow}, we have that our set of weight $k$ vectors $A_t$ and set of weight $k-1$ vectors given by $B_t = A_t \setminus \partial(A_t)$ satisfy $B_t \cap \partial(A_t) = \phi$. So, the collection $A_t \cup B_t$ is an antichain. In Algorithm~\epsSSMatrix, $U$ and $\tilde{U}$ contain the same collection of vectors, only differing in the ordering $\zeta$. From Lemma~\ref{lem:antichain_ss}, we have $U$ constructed from $A_t \cup B_t$ is a strongly separating matrix.
\end{proof}

The following lemma follows from LYM inequality in Lemma~\ref{lem:LYM_ineq}.
\begin{lemma}
Let $U_{\OPT}$ represent the optimum solution with $a^*_q$ representing the number of rows of $U$ with weight $q$. Then, for any $t \leq n$ :  
$$ \sum_{q= 1}^t a^*_q \leq \binom{m}{t}. $$
\end{lemma}
\begin{proof}
From Lemma~\ref{lem:kk_ss} and Corollary~\ref{cor:opt}, we know that the matrix $U_{\OPT}$ is a strongly separating matrix. Therefore, using Lemma~\ref{lem:ss_antichain}, we can construct a collection $\mathcal{T}$ defined over $\{1, 2, \cdots, m\}$ such that $\mathcal{T}$ is an antichain. $T_i \in \mathcal{T}$ corresponds to a row of $U_{\OPT}$ and $|T_i| = \norm{U_{\OPT}(i)}_1$ represents the weight of $i^{th}$ row of $U$. 
Applying LYM inequality from Lemma~\ref{lem:LYM_ineq} gives us:
\[\sum_{q=1}^t \frac{a^*_q}{{m \choose t}} \leq \sum_{q=1}^t \frac{a^*_q}{{m \choose q}} \leq \sum_{q=0}^m \frac{a^*_q}{{m \choose q}} \leq 1\]
and so $\sum_{q=1}^t a^*_q \leq \binom{m}{t}$.
\end{proof}

The next lemma gives an upper bound for $\binom{m}{t}$ that can be used to simplify the statement of the Theorem~\ref{thm:approx_1+e}.
\begin{lemma}
\label{lem:mt_lb}
If $6/k \leq \epsilon\leq 1/2$ and $m\geq 2\log_2 n$:
\[\binom{m}{t} \leq 2n \cdot 2^{-(\epsilon k/6) \log_2(m/(2k))}.\]
\end{lemma} 
\begin{proof}

By the definition of $k$, 
\[\binom{m}{t}=\binom{m}{k-1} \binom{m}{t}/ \binom{m}{k-1}< n  \binom{m}{t}/ \binom{m}{k-1}.\]
Let $H(x)$ denote the binary entropy function.
Note that $t = \lfloor k - \eps k/3 \rfloor$. Therefore,
\[
(k-1)-t \geq k-1 -k+\epsilon k/3=\epsilon k/3 -1\geq \epsilon k/6,
\]
and that for all $x\in [t/m,(k-1)/m]$,  \[H'(x)\geq H'\left (\frac{k-1}{m} \right )=\log_2\left  (\frac{m}{k-1}-1\right )\geq \log_2 \left (\frac{m}{2k} \right ),\]
where we used the assumption $t/m\leq (k-1)/m\leq 1/2$.
Hence,
\[
|H((k-1)/m)-H(t/m)| \geq \frac{\epsilon k/6}{m} \log_2 \left (\frac{m}{2k} \right ).
\]

\noindent Using the bound from (\cite{macwilliams1977theory}, Page 309)
\[
\sqrt{\frac{a}{8b(a-b)}}2^{mH(b/a)} \leq \binom{a}{b}\leq \sqrt{\frac{a}{2\pi b(a-b)}}2^{mH(b/a)},
\]
we get that 
\begin{eqnarray*}
\binom{m}{t} / \binom{m}{k-1}
& \leq & 2^{m (H(t/m)-H((k-1)/m)} 
\sqrt{\frac{8(k-1)(m-k+1)}{2\pi t(m-t)}} \leq  2\cdot 2^{-(\epsilon k/6) \log_2(m/(2k))}
\end{eqnarray*}
where the last inequality used  $\epsilon \leq 1 $.

\end{proof}

\begin{corollary}(Corollary~\ref{cor:maxcost} Restated). Algorithm~\epsSSMatrix is a $(1+\epsilon)$-approximation if the maximum cost satisfies 
\[
c_{max} \leq \epsilon/6 \cdot 2^{(\epsilon k/6) \log_2(m/(2k))}
\]
assuming $n^{\epsilon/6} \geq m\geq 2\log_2 n$.
If a) $m\geq (2\log_2 n)^{c_1}$ for some constant $c_1>1$ or b) $4\log_2 n\leq m\leq c_2 \log_2 n$ for some constant $c_2$ then the RHS  bound is at least $\epsilon/6 \cdot n^{\Omega(\epsilon)}$. 
\end{corollary}
\begin{proof}
First note that $k\geq \log_m n$ since 
\[
m^k \geq \binom{m}{k}\geq n. \  
\]

\noindent When $2\log n \leq m \leq n^\epsilon$, we have $k \geq \log_m n \geq \frac{6}{\epsilon}$. From the previous lemma~\ref{lem:mt_lb},
\[
c_{max} \leq   \epsilon/6 \cdot 2^{(\epsilon k/6) \log_2(m/(2k))} \leq \epsilon n/3\binom{m}{t}
\] 
Using Theorem~\ref{thm:approx_1+e}, we have that Algorithm~\epsSSMatrix~is a $(1+\eps)$-approximation. We next consider the simplification in Part (a). If $m\geq (2\log_2 n)^{c_1}$ for some $c_1 > 1$ then $k\leq \log_2 n$ as $\binom{m}{\log n} \geq n$. So $2k\leq 2\log_2 n \leq m^{1/{c_1}} $. Hence, 
\[\log_2 (m/(2k))\geq (1-1/{c_1}) \log_2 m\] and so
\[
2^{(\epsilon k/6) \log_2(m/(2k))} \geq 
2^{(\epsilon k (1-1/{c_1}) (\log_2 m )/6)} \geq n^{\frac{\epsilon(1-1/{c_1})}{6} }.
\]
where the last inequality follows since $k\geq \log_m n$.

\noindent We next consider the simplification in Part (b). Now suppose  $m\leq c_2 \log_2 n$ for some constant $c_2\geq 2$ 
then, $k\geq \log_{ec_2} n$ since 
\[
(c_2 e)^k \geq (me/k)^k \geq \binom{m}{k}\geq n \  . 
\]

\noindent Note that for $m\geq 4\log n$,  
\[\log_2(m/(2k))\geq \log_2(4\log n/(2\log_{2} n))\geq 1\]
and so 
\begin{eqnarray*}
2^{(\epsilon k/6) \log_2(m/(2k))}  \geq   
2^{(\epsilon k/6)}   \geq & n^{\frac{\epsilon}{6 \log_{ec_2} 2} } \ . 
\end{eqnarray*}
\end{proof}

\section{Missing Details from Section~\ref{sec:unweighted}} \label{app:unweighted}
\noindent\textbf{Removing Dependence on $\tau$ in Algorithms~\RecoverG,~\LatentsNEdges~ and~\LatentsWEdges.} 
Let $\GGG$ be a $\tau$-causal graph. Algorithms~\RecoverG,~\LatentsNEdges, and~\LatentsWEdges~ assume that we know $\tau$, however this assumption can be easily removed. For a fixed $\tau$, let $\GGG_\tau$ be graph returned after going through all these above algorithms. Given $\GGG_\tau$, checking whether $v_k$ is a $p$-collider for some pair $v_i,v_j$ is simple, iterate over all paths between $v_i$ and $v_j$ that include $v_k$. Let $\Pi = \{\pi_1,\dots,\pi_r\}$ be these paths. For each $\pi_w \in \Pi$, remove the edges in $\pi_w$ from $\GGG_\tau$ see if $v_k$ has a descendant in $\Pa(v_i) \cup \Pa(v_j)$ in this modified graph. If this holds for any path $\pi_w \in \Pi$, then $v_k$ is a $p$-collider for the pair $v_i,v_j$. We describe an efficient algorithm for finding $p$-colliders in section~\ref{app:experiments}.

The idea is as follows, we invoke Algorithms~\RecoverG,~\LatentsNEdges~and\\ \LatentsWEdges~for $\tau=1,2,4,..$, until we find the first $\hat{\tau}$ and $2\hat{\tau}$ such that $\GGG_{\hat{\tau}} = \GGG_{2\hat{\tau}}$. We now check whether the observable nodes in $\GGG_{\hat{\tau}}$ has at most $\hat{\tau}$ $p$-colliders, if so we are output $\GGG_{\hat{\tau}}$ (and $\hat{\tau}$). Otherwise, we continue by doubling $\tau$, i.e., by considering $2\hat{\tau}$ and $4\hat{\tau}$. By increasing $\tau$ by a constant factor, it is easy to see that process will stop in at most $\log (2\tau)$ steps and when it stops it produces the correct observable graph $\GGG$ and also that $\hat{\tau} \leq 2 \tau$.  Overall, this will increase the number of interventions in Theorem~\ref{thm:lat} by a factor of $O(\log \tau)$ (to $O(\tau^2 \log n\log \tau + n\tau \log n\log \tau)$ interventions). Through a union bound, the same success probability of $1-O(1/n^2)$ can be ensured by adjusting the constants.

\begin{lemma} [Lemma~\ref{lem:indep} Restated]
Let $v_i \in Anc(v_j)$. $v_i \indep v_j \mid \doo(v_i \cup P_{ij}), \Anc(v_j) \setminus \{ v_i \}$ iff $v_i \not \in \Pa(v_j)$.
\end{lemma}
\begin{proof}
Suppose $v_i \in \Anc(v_j) \setminus \Pa(v_j)$. Consider the interventional distribution $\doo(v_i \cup P_{ij})$ where $P_{ij}$ is the set of $p$-colliders between $v_i$ and $v_j$. We intervene on $v_i$ to block the path (if present) given by $v_i \leftarrow \tilde{l} \rightarrow v_j$ where $\tilde{l} \in \lat$. Consider all the remaining \emph{undirected} paths between $v_i$ and $v_j$ denoted by $\Pi_{ij}$. We divide $\Pi_{ij}$ into three cases. Let $\pi \in \Pi_{ij}$ be a path from $v_i$ to $v_j$.
\begin{itemize}
\item[1.] $\pi$ contains no \emph{colliders}, then, $\pi$ is blocked by $\Anc(v_j) \setminus \{ v_i \}$. As $\pi$ contains no colliders, we can write $\pi = v_i \cdots v_k \rightarrow v_j$ where $v_k \in \Anc(v_j)$. As we are conditioning on $\Anc(v_j) \setminus \{ v_i \} \supseteq \{ v_k \}$, $\pi$ is blocked by $v_k$.
\item[2.] $\pi$ contains colliders but not a $p$-collider. We argue that there are no collider nodes in $\pi$ that are also in $\Anc(v_j) \setminus \{ v_i \}$. As there are no $p$-colliders, it means that all the colliders have no descendants in the conditioning set $\Anc(v_j) \setminus \{ v_i \}$. Because if a collider $v_c$ have a descendant in $\Anc(v_j) \setminus \{ v_i \}$, then there is a path from $v_c$ to $\Pa(v_j)$ through $\Anc(v_j) \setminus \{ v_i \}$. This means that $v_c$ is a $p$-collider, contradicting our assumption. Therefore, from Rule-2 of $d$-separation, $\pi$ is blocked.
\item[3.] $\pi$ contains at least one $p$-collider. We are intervening on $P_{ij}$ containing all the $p$-colliders. In the intervened mutilated graph, all the $p$-colliders no longer have an incoming arrow and therefore are not colliders. So $\pi$ is blocked.
\end{itemize}
If $v_i \not \in \Pa(v_j)$, we can conclude that $v_i \indep v_j \mid \doo(\{v_i\} \cup  P_{ij}), \Anc(v_j) \setminus \{ v_i \}$.
Suppose $v_i \in \Pa(v_j)$. In the interventional distribution $\doo(\{v_i\} \cup  P_{ij})$, we still have $v_i \in \Pa(v_j)$ and any conditioning will not block the path $\pi = v_i \rightarrow v_j$. Therefore, $v_i \notindep v_j \mid (\doo(\{v_i\} \cup  P_{ij}), \Anc(v_j) \setminus \{ v_i \}$ if $v_i \in \Pa(v_j)$. 
\end{proof}

\begin{lemma}
\label{lem:graphwhp}
Let $\GGG(V \cup L, E \cup E_L)$ be a $\tau$-causal graph with observable graph $G(V,E)$.
Given an ancestral graph $\Anc(G)$, Algorithm~\RecoverG~correctly recovers all edges in the observable graph with probability at least $1-1/n^2$.
\end{lemma}
\begin{proof}
Let $\tau'=\max\{\tau,2\}$.
From Lemma~\ref{lem:indep}, we can recover the edges of $G$ provided we know the $p$-colliders between every pair of nodes. As we do not know the graph $G$, we devise a randomized strategy to hit all the $p$-colliders, whilst ensuring that we don't create a lot of interventions. Suppose $\text{max}_{(v_i,v_j) \in V \times V} |P_{ij} | \leq \tau$. We show that with high probability, $\forall v_i \in Anc(v_j),\, \exists A_t \text{ such that } \{ v_i \} \cup P_{ij}  \subseteq  A_t$ and $v_j \not \in A_t$. We can then use the CI-test described in Lemma~\ref{lem:indep} to verify whether $v_i$ is a parent of $v_j$. In Algorithm~\RecoverG, we repeat this procedure on every edge of $Anc(G)$ and output $G$.
 
Suppose $v_i \in Anc(v_j)$. Let $\Gamma_t$ denote the event that $A_t \in \AAA_\tau$ such that $\{ v_i \} \cup P_{ij} 
\subseteq A_t $ and $v_j \not \in A_t$ for a fixed $t \in \{1,\dots,72\tau' \log n \}$. Let $T=72\tau' \log n$. As we include a vertex $v_i \in A_t$ with probability $1-1/\tau'$, we obtain
\begin{align*}
\Pr[\Gamma_t] &= \left( 1 - \frac{1}{\tau'} \right)^{|P_{ij}| + 1} \frac{1}{\tau'} \geq  \left( 1 - \frac{1}{\tau'} \right)^{\tau' + 1} \frac{1}{\tau'}.
\end{align*}
Using the inequality $(1+\frac{x}{n})^n \geq e^{x}(1-\frac{x^2}{n})$ for $|x| \leq n$, and since $\tau' \geq 2$ we have: 
\begin{eqnarray*}
& \Pr[\Gamma_t] \geq  \frac{1}{\e^{\tau'+1/\tau'}} ( 1-\frac{(\tau'+1)^2}{\tau'^2(\tau'+1)} ) \frac{1}{\tau'} \geq  \frac{1}{18 \tau'} & \\
& \Rightarrow  \Pr[\bar{\Gamma}_t]   \leq 1 - \frac{1}{18\tau'} \mbox{ and } \Pr[\exists t \in [T] : \Gamma_t]  \geq  1- \left( 1 - \frac{1}{18\tau'} \right)^{72 \tau'\log n}.&
\end{eqnarray*}
Using the inequality $(1+\frac{x}{n})^n \leq e^{x}$ for $|x| \leq n$ we have:  
\begin{align*}
\Pr[\exists t \in [T] : \Gamma_t]  &\geq  1- \frac{1}{n^4}. 
\end{align*}
So the probability that there exists at least one set $A_t$ for the given pair $v_i,v_j$ for which $v_i \cup P_{ij} \subseteq  A_t$ and $v_j \not \in A_t$ is at least $1-\frac{1}{n^4}$.\footnote{Note by adjusting the constant $72$, we could have pushed this probability to any $1/n^c$ for constant $c$.} To ensure this probability of success for every pair of variables, we use a union bound over the $n^2$ node pairs.	 			\end{proof}

\begin{proposition} [Proposition~\ref{prop:obs} Restated]
Let $\GGG(V \cup L, E \cup E_L)$ be a $\tau$-causal graph with observable graph $G(V,E)$. There exists a procedure to recover the observable graph using $O(\tau \log n + \log n)$ many interventions with  probability at least $1-1/n^2$.
\end{proposition}
\begin{proof}
As is well-known, e.g.~\citep{neurips17}, a strongly separating set system can be constructed with $m=2\log n$ interventions by using the binary encoding of the numbers $1,\dots,n$. Two intervention sets are constructed for every bit location $k \in [\log n]$, one with any node $v_i$ if the number $i$ has $k$th bit set to 1, and other with any node $v_i$ if the number $i$ has $k$th bit set to 0. Therefore, we require $2 \log n$ interventions to obtain ancestral graph $\Anc(G)$ of the observable graph. From Lemma~\ref{lem:graphwhp}, we require $O(\tau \log n)$ interventions to recover all the edges of observable graph of $G$ from $\Anc(G)$ with probability $1-\frac{1}{n^2}$. Therefore, using $O(\tau\log n)$ interventions, Algorithm \RecoverG can recover the observable graph $G(V, E)$ with high probability. 
\end{proof}

It is well established that $\log(\chi(G))$ interventions are necessary and sufficient in the causally sufficient systems (where there are no latents) where $\chi(G)$ is the chromatic number of $G$. Generalized over all graphs this becomes $\log(n)$. Our following lower bound shows that, even if there are no latent variables in the underlying system, if the algorithm cannot rule latents out, and needs to consider latents as a possibility to compute the graph skeleton, then $\Omega(n)$ interventions are necessary.~\citet{shanmugam2015learning} provide a lower bound in a different setting, when the intervention sets are required to have only limited number of variables. 

\begin{proposition}[Proposition~\ref{prop:lb} Restated]
There exists a graph causal $\GGG(V \cup \lat, E \cup E_L)$ such that every non-adaptive algorithm requires $\Omega(n)$ many interventions to recover even the observable graph $G(V,E)$ of $\GGG$. 
\end{proposition} 
\begin{proof}
Consider an ordering of observable variables given by $v_1, v_2, \cdots, v_n$. Let $G$ be a graph with all directed edges $(v_a,v_b)$ for all $b > a$. Suppose the set of interventions generated by the non-adaptive algorithm is given by $\mathcal{H}$. Now consider $v_i$ for some fixed $i \geq \frac{n}{4}$. 

We claim that if every intervention $H \in \mathcal{H}$ is such that for some $j \in \{3, \cdots, i-1 \},\, v_j \not \in H$, then there exists a graph $G_i$ such that $G$ and $G_i$ are both indistinguishable under all the interventions in $\mathcal{H}$ irrespective of other conditioning. Now consider any set $H_j \subseteq (\{ v_1, v_2, \cdots, v_{i-1} \} \setminus \{v_j\}) \bigcup \{v_{i+1},\cdots,v_n\}$.   Let $G_i$ be such that it contains all the directed edges $(v_a,v_b)$ for all $b > a$ but does not contain the directed edge $(v_1, v_i)$.
To distinguish between $G$ and $G_i$ one needs to determine whether $v_1 \rightarrow v_i$. Note that any intervention we use to determine the edge should contain $v_1$ to rule out the possibility of the influence of latent $v_1 \leftarrow l_{1i} \rightarrow v_i$ on the CI-tests we perform. Now, under $\doo(H_j)$, there are only two CI-tests possible to determine whether $v_1 \rightarrow v_i$ : $v_1 \indep v_i \mid v_j, \doo(H_j)$ and $v_1 \indep v_i \mid \doo(H_j)$. However, for both graphs $G$ and $G_i$, both these independence tests will always turn out negative. In the former case, it is because $v_j$ will be a collider on the path $v_i,v_j,v_{j-1},v_i$, and in the latter case there is a path $v_1,v_j,v_i$ that is not blocked. In other words, the CI-tests will provide no information to distinguish between $G$ and $G_i$, unless $\mathcal{H}$ contains the set $\{v_1, v_3,\dots,v_{i-1}\}$.

One can similarly construct these $G_i$'s for all $i \geq \frac{n}{4}$, thereby $\mathcal{H}$ needs to contain the intervention sets $\{v_1, v_3,\dots,v_{i-1}\}$ for all $n/4 \leq i \leq n$ to separate $G$ from all the $G_i$'s. This proves the claim. 
\end{proof}

\subsection{Latents Affecting Non-adjacent Nodes in $G$}
\label{app:latent_non_adj}
\begin{algorithm}[t]
\begin{small}
\caption{\LatentsNEdges$(G(V, E), \DDD_\tau)$}
\label{alg:latentNE}
\begin{algorithmic}[1]
\State  $\lat \leftarrow \phi, E_L \leftarrow \phi$
\For{ $(v_i,v_j) \not \in {E}$}
\State  Let $\mathcal{D}_{ij} = \{ D \mid D \in \DDD_\tau \text{ and } v_i, v_j \not \in D \}$
\If{ $v_i \notindep v_j \mid do(D) \cup \Pa(v_i) \cup \Pa(v_j)$ for every $D \in \mathcal{D}_{ij}$}
\State  $\lat \leftarrow \lat \cup {l_{ij}}, E_L \leftarrow E_L \cup \{ (l_{ij}, v_i), (l_{ij}, v_j) \}$
\EndIf
\EndFor
\State  \textbf{return} $\GGG(V \cup \lat, E \cup E_L)$
\end{algorithmic}
\end{small}
\end{algorithm}
Let $\bar{E} =\{ (v_i, v_j) \mid (v_i,v_j) \not \in E \}$ be the set of non-edges in $G$. The entire procedure for finding latents between non-adjacent nodes in $G$ is described in Algorithm~\LatentsNEdges. Similar to Algorithm \RecoverG, we block the paths by conditioning on parents and intervening on $p$-colliders. The idea is based on the observation that for any non-adjacent pair $v_i, v_j$ an intervention on the set $P_{ij}$ and conditioning on the parents of $v_i$ and $v_j$ will make $v_i$ and $v_j$ independent, unless there is a latent between then. The following lemma formalizes this idea.

\begin{lemma}
\label{lem:blocking}
Suppose $(v_i, v_j) \in \bar{E}$. Then, $v_i \indep v_j \mid \doo(P_{ij}), \Pa(v_i) \cup \Pa(v_j)$ iff $v_i$ and $v_j$ has no latent between them.
\end{lemma}
\begin{proof}

Suppose there is no latent between $v_i$ and $v_j$. We follow the proof similar to the Lemma~\ref{lem:indep}. Consider the pair of variables $v_i$ and $v_j$ and all the paths between them $\Pi_{ij}$. Let $\pi \in \Pi_{ij}$.
\begin{itemize}
\item[1.] Let $\pi$ be a path not containing any colliders. Using Rule-1 of $d$-separation, we can block $\pi$ by conditioning on either $\Pa(v_i)$ or $\Pa(v_j)$.
\item[2.] If $\pi$ contains colliders and no $p$-colliders, then, using Rule-2 of $d$-separation, $\pi$ is blocked as the colliders have no descendants in $\Pa(v_i) \cup \Pa(v_j)$.
\item[3.] We block the paths $\pi$ containing $p$-colliders by intervening on $P_{ij}$
\end{itemize}  
As all the paths in $\Pi_{ij}$ are blocked, we have $v_i \indep v_j \mid \doo(P_{ij}), \Pa(v_i) \cup \Pa(v_j)$. 
If there is a latent $l_{ij}$ then the path $ v_i \leftarrow l_{ij} \rightarrow v_j$ is not blocked and therefore $v_i \notindep v_j \mid (\doo(P_{ij}), \Pa(v_i) \cup \Pa(v_j))$.
\end{proof}

Formally, let $D_t \subseteq V$ for $t \in \{1, 2, \cdots, 24\tau'^2\log n\}$ be constructed by including  every variable $v_i \in V$ with probability $1-\frac{1}{\tau'}$ where $\tau'=\max\{\tau,2\}$. Let $\DDD_\tau = \{D_1,\cdots,D_{24\tau'^2 \log n}\}$ be the collection of the set $D_t$'s. Using these interventions $\DDD_\tau$, we argue that we can recover all the latents between non-edges of $G$ correctly with high probability.

\begin{proposition}[Proposition~\ref{prop:1} Restated]
Let $\GGG(V \cup L, E \cup E_L)$ be a $\tau$-causal graph with observable graph $G(V,E)$.
Algorithm~\LatentsNEdges~with $O(\tau^2 \log n + \log n)$ many interventions recovers all latents effecting pairs of non-adjacent nodes in the observable graph $G$ with probability at least $1 - 1/n^2$. 
\end{proposition}
\begin{proof}
We follow a proof similar to Lemma~\ref{lem:graphwhp}. Consider a pair of variables $v_i$ and $v_j$ such that there is no edge between them in $G$. From Lemma~\ref{lem:blocking}, we know that by intervening on all the colliders between $v_i$ and $v_j$, we can identify the presence of a latent. In Algorithm~\LatentsNEdges, we iterate over sets in $\mathcal{D}_{ij}$. As $\mathcal{D}_{ij} \subseteq \DDD_\tau$, we have $|\mathcal{D}_{ij}| \leq 24\tau'^2 \log n$.
Let $\Gamma_t$ denote the event that $D_t \in \mathcal{D}_{ij}$ is such that $v_i, v_j \not \in D_t$ and $P_{ij} \subseteq D_t$ for a fixed $t \in \{1,\cdots,24\tau'^2 \log n\}$. Let $T=24 \tau'^2 \log n$.
\begin{align*}
\Pr[\Gamma_t] &= \left( 1 - \frac{1}{\tau'} \right)^{|P_{ij}|} \frac{1}{\tau'^2} \geq  \left( 1 - \frac{1}{\tau'} \right)^{\tau'} \frac{1}{\tau'^2}. 
\end{align*}
Using the inequality $(1+\frac{x}{n})^n \geq e^{x}(1-\frac{x^2}{n})$ for $|x| \leq n$, and since $\tau' \geq 2$ we have: 
\begin{eqnarray*}
& \Pr[\Gamma_t] \geq  \frac{1}{\e} ( 1-\frac{1}{\tau'} ) \frac{1}{\tau'^2} \geq  \frac{1}{2\e \tau'^2} & \\
& \Rightarrow  \Pr[\bar{\Gamma}_t]   \leq 1 - \frac{1}{6\tau'^2} \mbox{ and } \Pr[\exists t \in [T] : \Gamma_t]  \geq  1- \left( 1 - \frac{1}{6\tau'^2} \right)^{24 \tau'^2\log n}.&
\end{eqnarray*}
Using the inequality $(1+\frac{x}{n})^n \leq e^{x}$ for $|x| \leq n$ we have:  
\begin{align*}
\Pr[\exists t \in [T] :\Gamma_t]  &\geq  1- \frac{1}{n^4}. 
\end{align*}
So the probability that there exists a set $D_t$ for which $v_i, v_j \not \in D_t$ and $P_{ij} \subseteq D_t$ is at least $1-\frac{1}{n^4}$. A union bound over at most $n^2$ pair of variables completes the proof.
\end{proof}

\subsection{Latent Affecting Adjacent Nodes in $G$}
We follow an approach similar to the one presented in section~\ref{app:latent_non_adj} for detecting the presence of latent between an edge $v_i \rightarrow v_j$ in $G$. In Algorithm~\ref{alg:latentWE}, we block all the paths (excluding the edge)  between the variables $v_i$ and $v_j$ using a conditioning set $\Pa(v_j)$ in the intervention distribution $\doo(\Pa(v_i) \cup P_{ij})$ in the {\emph{do-see}} tests we perform. This idea is formalized using the following lemma. 
\begin{lemma}
\label{lem:latblocking}
Suppose $v_i \rightarrow v_j \in G$. Let $l_{tj}$ be a latent between $v_t$ and $v_j$ where $v_t \neq v_i$ and $v_i, v_j \not \in B$, $P_{ij} \subseteq B$.  Then, $l_{tj} \indep v_i \mid  \Pa(v_j), \doo(B \cup \{v_i\} \cup \Pa(v_i))$ and $l_{tj} \indep v_i \mid \Pa(v_j), \doo(\Pa(v_i) \cup B)$. 
\end{lemma}
\begin{proof}
The proof goes through by analyzing various cases. We give a detailed outline of the proof.

 Claim 1: $l_{tj} \indep v_i \mid \Pa(v_j), \doo(B \cup \{v_i\} \cup \Pa(v_i))$. 
 Suppose $v_t \in \Pa(v_i) \cup B$. Consider all the paths between $v_i$ and $l_{tj}$ in the interventional distribution $ \doo(B \cup \{v_i\} \cup \Pa(v_i))$. The only paths that are not separated because of the intervention are $l_{tj} \rightarrow v_j \leftarrow v_i$, $l_{tj} \rightarrow v_j \leftarrow v_k \cdots \leftarrow v_i$ where $v_k \in \Pa(v_j)$, and $l_{tj} \rightarrow v_j \rightarrow \cdots \leftarrow v_i$. As we are not conditioning on $v_j$, $l_{tj} \rightarrow v_j \leftarrow v_i$ is blocked (Rule-2 in $d$-separation); conditioning on $\Pa(v_j) \ni v_k$ block the paths $l_{tj} \rightarrow v_j \leftarrow v_k \cdots \leftarrow v_i$ (Rule-1 in $d$-separation); and $l_{tj} \rightarrow v_j \rightarrow \cdots \leftarrow v_i$ paths have a collider that is not $\Pa(v_j)$ hence blocked by Rule-2 in $d$-separation.
 
Suppose $v_t \not \in \Pa(v_i) \cup B$. As before it follows that all paths between $l_{tj}$ and $v_i$ going through $v_j$ are blocked. All other paths between $l_{tj}$ and $v_i$ should have a collider. This is because  in any such path $\pi$ the only edge from $l_{tj}$ is $l_{tj} \rightarrow v_t$ and the edges that remain at $v_i$ are outgoing. It is easy to see that the collider on this path $\pi$ can't be in $\Pa(v_j)$ because otherwise it will also be a $p$-collider between $v_i$ and $v_j$ which are intervened on through $B$. 
When there is a collider on the path that is not in the conditioning set, then the path is blocked (Rule-2 in $d$-separation). The same holds for all paths between $l_{tj}$ and $v_i$.

Claim 2: $l_{tj} \indep v_i \mid \Pa(v_j), \doo(B \cup \Pa(v_i))$. Consider all the paths between $l_{tj}$ and $v_i$. Using the above arguments, we have that all paths containing $v_j$ are blocked. All other paths between $l_{tj}$ and $v_i$ should have a collider. This is because in any such path $\pi$ the only edge from $l_{tj}$ is $l_{tj} \rightarrow v_t$ and $\pi$ will end at $v_i$ either as $l_{tj} \cdots \leftarrow v_i$ or $l_{tj} \rightarrow \cdots \leftarrow v_k \rightarrow v_i$ where $v_k \in \Pa(v_i)$. It is again easy to see that the collider on this path $\pi$ can't be in $\Pa(v_j)$ because otherwise it will also be a $p$-collider between $v_i$ and $v_j$ which are intervened on through $B$. As before, when there is a collider on the path that is not in the conditioning set, then the path is blocked (Rule-2 in $d$-separation). The same holds for all paths between $l_{tj}$ and $v_i$.
\end{proof}

\begin{lemma} [Lemma~\ref{lem:first} Restated]
Suppose $v_i \rightarrow v_j \in G$ and $v_i, v_j \not \in B$, and $P_{ij} \subseteq B$ then, $\Pr[v_j \mid v_i, \Pa(v_j), \doo(\Pa(v_i) \cup B)] = \Pr[v_j \mid  \Pa(v_j), \doo(\{v_i\} \cup \Pa(v_i) \cup B)]$ if there is no latent $l_{ij}$ with $ v_i \leftarrow l_{ij} \rightarrow v_j$.
\end{lemma}
\begin{proof}
Suppose $v_i \rightarrow v_j$ in $G$ and there is no latent between $(v_i,v_j)$. Then, we claim that $\Pr[v_j \mid v_i, \Pa(v_j), \doo(\Pa(v_i) \cup B)]  = \Pr[v_j \mid  \Pa(v_j), \doo(\{v_i\} \cup  \Pa(v_i) \cup B)]$. Let $L_j$ represents all the latent parents of $v_j$.
By including $v_i$ in the intervention,
\begin{align} \label{eqn:1}
& \Pr[v_j \mid \Pa(v_j), \doo(\{v_i\} \cup  \Pa(v_i) \cup B)]  \nonumber\\
& = \sum_{L_j} \Pr[v_j \mid L_j, \Pa(v_j), \doo(\{v_i\} \cup  \Pa(v_i) \cup B)] \Pr[L_j \mid \Pa(v_j), \doo(\{v_i\} \cup  \Pa(v_i) \cup B)]. \nonumber\\
& = \sum_{L_j} \Pr[v_j \mid L_j, \Pa(v_j), \doo(\{v_i\} \cup  \Pa(v_i) \cup B)] \Pr[L_j \mid \Pa(v_j), \doo(\Pa(v_i) \cup B)].
\end{align}
As the value of $L_j$ is only affected by conditioning on its descendants, and in the interventional distribution $\doo(\Pa(v_i))$, $v_i$ is not a descendant of $L_j$, the last statement is true.

Under conditioning on $v_i$
\begin{align} \label{eqn:2}
& \Pr[v_j \mid v_i, \Pa(v_j), \doo(\Pa(v_i) \cup B)] \nonumber\\
&= \sum_{L_j} \Pr[v_j \mid L_j, v_i, \Pa(v_j), \doo(\Pa(v_i) \cup B)]\Pr[L_j \mid v_i, \Pa(v_j), \doo(\Pa(v_i) \cup B)] \nonumber\\
&=  \sum_{L_j} \Pr[v_j \mid L_j, v_i, \Pa(v_j), \doo( \Pa(v_i) \cup B)]\Pr[L_j \mid \Pa(v_j), \doo(\Pa(v_i) \cup B)].
\end{align}
The last statement is true because $L_j \indep v_i \mid \Pa(v_j)$ in the distribution $\doo(B \cup \Pa(v_i))$ from Lemma~\ref{lem:latblocking}.
{From the invariance principle} (page 24 in~\citep{pearl}, \cite{neurips17}), we have for any variable $v_i$ 
$$ \Pr[v_i \mid \Pa(v_i)] = \Pr[v_i \mid Z, \doo(\Pa(v_i) \setminus Z)] \text{ for any }Z \subseteq \Pa(v_i) $$
Applying it to our case we get $$\Pr[v_j \mid L_j, v_i, \Pa(v_j), \doo(\Pa(v_i) \cup B)] = \Pr[v_j \mid L_j, \Pa(v_j), \doo(\{v_i\} \cup  \Pa(v_i) \cup B)].$$

Putting this together with~\eqref{eqn:1} and~\eqref{eqn:2}, we get $\Pr[v_j \mid v_i, \Pa(v_j), \doo(\Pa(v_i) \cup B)] = \Pr[v_j \mid  \Pa(v_j), \doo(\{v_i\} \cup  \Pa(v_i) \cup B)]$, if there is no latent $l_{ij}$ with $ v_i \leftarrow l_{ij} \rightarrow v_j$.
\end{proof}

\begin{lemma} [Lemma~\ref{lem:second} Restated]
Suppose $v_i \rightarrow v_j \in G$ and $v_i, v_j \not \in B$, and $P_{ij} \subseteq B$, then, $\Pr[v_j \mid v_i, \Pa(v_j), \doo(\Pa(v_i) \cup B)] \neq \Pr[v_j \mid  \Pa(v_j), \doo(\{v_i\} \cup \Pa(v_i) \cup B)]$ if there is a latent $l_{ij}$ with $ v_i \leftarrow l_{ij} \rightarrow v_j$.
\end{lemma}
\begin{proof}
Suppose $v_i \rightarrow v_j$ in $G$ and there is a latent $l_{ij}$ between $(v_i,v_j)$. Then, we claim that $\Pr[v_j \mid v_i, \Pa(v_j), \doo(\Pa(v_i) \cup B)]  \neq \Pr[v_j \mid  \Pa(v_j), \doo(\{v_i\} \cup  \Pa(v_i) \cup B)]$. Let $L_j$ represents all the latent parents of $v_j$, where $l_{ij} \in L_j$. Therefore, $v_i$ is a descendant of $L_j$. 
By including $v_i$ in the intervention,
\begin{align*} 
& \Pr[v_j \mid v_i ,\Pa(v_j), \doo(\Pa(v_i) \cup B)]  \nonumber\\
& = \sum_{L_j} \Pr[v_j \mid L_j, \Pa(v_j), \doo(\{v_i\} \cup  \Pa(v_i) \cup B)] \Pr[L_j \mid \Pa(v_j), \doo(\{v_i\} \cup  \Pa(v_i) \cup B)]\\
&= \sum_{L_j} \Pr[v_j \mid L_j, \Pa(v_j), \doo(\{v_i\} \cup  \Pa(v_i) \cup B)] \Pr[L_j \mid \Pa(v_j), \doo( \Pa(v_i) \cup B)].
\end{align*}
As the value of $L_j$ is only affected by conditioning on its descendants, and in the interventional distribution $\doo(\Pa(v_i))$, $v_i$ is not a descendant of $L_j$, the last statement is true. Under conditioning on $v_i$, we have : 
\begin{align*} 
& \Pr[v_j \mid v_i, \Pa(v_j), \doo(\Pa(v_i) \cup B)] \nonumber\\
&= \sum_{L_j} \Pr[v_j \mid L_j, v_i, \Pa(v_j), \doo(\Pa(v_i) \cup B)]\Pr[L_j \mid v_i, \Pa(v_j), \doo(\Pa(v_i) \cup B)] \nonumber\\
&= \sum_{L_j} \Pr[v_j \mid L_j, v_i, \Pa(v_j), \doo(\Pa(v_i) \cup B)] \frac{\Pr[v_i \mid L_j, \Pa(v_j), \doo(\Pa(v_i) \cup B)]}{Pr[v_i \mid \Pa(v_j), \doo(\Pa(v_i) \cup B) ]} \Pr[L_j \mid \Pa(v_j), \doo(\Pa(v_i) \cup B)] \nonumber\\
&=  \sum_{L_j} \Pr[v_j \mid L_j, v_i, \Pa(v_j), \doo(\Pa(v_i) \cup B)] \frac{\Pr[v_i \mid L_j, \Pa(v_j), \doo(\Pa(v_i) \cup B)]}{Pr[v_i \mid \Pa(v_j), \doo(\Pa(v_i) \cup B) ]} \Pr[L_j \mid \Pa(v_j), \doo(\Pa(v_i) \cup B)].
\end{align*}
{From the invariance principle} (page 24 in~\citep{pearl}, \cite{neurips17}), we have for any variable $v_i$ 
$$ \Pr[v_i \mid \Pa(v_i)] = \Pr[v_i \mid Z, \doo(\Pa(v_i) \setminus Z)] \text{ for any }Z \subseteq \Pa(v_i) $$
Applying it to our case we get $$\Pr[v_j \mid L_j, v_i, \Pa(v_j), \doo(\Pa(v_i) \cup B)] = \Pr[v_j \mid L_j, \Pa(v_j), \doo(\{v_i\} \cup  \Pa(v_i) \cup B)].$$

However, since the numerator of $$\frac{\Pr[v_i \mid L_j, \Pa(v_j), \doo(\Pa(v_i) \cup B)]}{\Pr[v_i \mid \Pa(v_j), \doo(\Pa(v_i) \cup B) ]}$$ 
depends on $L_j$ as $v_i$ is a descendant of $l_{ij} \in L_j$, whereas the denominator is not dependent on $L_j$, the ratio is not equal to $1$ unless in pathological cases. A similar situation arises in the do-see test analysis for~\citep{neurips17}. Hence, we have $\Pr[v_j \mid v_i, \Pa(v_j), \doo(\Pa(v_i) \cup B)] \neq \Pr[v_j \mid  \Pa(v_j), \doo(\{v_i\} \cup  \Pa(v_i) \cup B)]$.
\end{proof}

\begin{proposition} [Proposition~\ref{prop:2} Restated]
Let $\GGG(V \cup L, E \cup E_L)$ be a $\tau$-causal graph with observable graph $G(V,E)$.
Algorithm~\LatentsWEdges~with $O(n\tau \log n + n \log n)$ many interventions recovers all latents effecting pairs of adjacent nodes in the observable graph $G$ with probability at least $1-\frac{1}{n^2}$.
\end{proposition}
\begin{proof}
From Lemma~\ref{lem:graphwhp}, we know that with probability $1- \frac{1}{n^2}$, for every pair $v_i$ and $v_j$, there exists, with high probability, an intervention $B \in \BBB_\tau$ such that $v_i \in B, v_j \not \in B$ and $P_{ij} \subseteq B$. On this $B$,  using Lemmas~\ref{lem:first} and~\ref{lem:second}, we can identify the latent by using a distribution test on $B \cup Pa(v_i)$ and $B \cup Pa(v_i) \cup \{ v_i \}$.

For every variable $v_i \in V$, our algorithm constructs at most $2|\BBB_\tau|$ many interventions, given by $\doo(\{v_i\} \cup \Pa(v_i) \cup B)$ and $\doo(\Pa(v_i) \cup B)$ for every $B \in \BBB_\tau$. Therefore, the total number of interventions used by Algorithm~\LatentsWEdges~is $O(n\tau \log n + n \log n)$.  
\end{proof}

\section{Experiments} \label{app:experiments}
In this section, we compare the total number of interventions required to recover causal graph $\GGG$ parameterized by $p$-colliders (See section~\ref{sec:unweighted}) vs.\ maximum degree utilized by~\citep{neurips17}. 

\vspace{2mm}
\noindent\textbf{Setup}.
We demonstrate our results by considering sparse random graphs generated from the families of:  (i) Erd\"os-R\'enyi random graphs $G(n,c/n)$ for constant $c$, (ii) Random Bipartite Graphs generated using $G(n_1,n_2,c/n)$ model, with partitions $L$, $R$ and edges directed from $L$ to $R$, (iii) Directed Trees with degrees of nodes generated from power law distribution. In each of the graphs we generate, we additionally include latent variables by sampling $5\%$ of $\binom{n}{2}$ pairs and adding a latent between them. 

\vspace{2mm}
\noindent\textbf{Finding $p$-colliders.}
Let $\GGG$ contain observable variables and the latents. To find $p$-colliders between every pair of observable nodes of $\GGG$, we enumerate all paths between them and check if any of the observable nodes on a path can be a possible $p$-collider. As this became practically infeasible for larger values of $n$, we devise an algorithm that runs in polynomial time (in the size of the graph) by constructing an appropriate flow network and finding maximum flow in this network. We will first describe a construction that takes three nodes $(v_i, v_j, v_k)$ as input and checks if $v_k$ is a $p$-collider for the pair of nodes $v_i$ and $v_j$. Iterating over all possible nodes $v_k$ gives us all the $p$-colliders for the pair $v_i, v_j$.

\vspace{2mm}
\noindent \textsc{Construction}. 
 If $v_k$ is not an ancestor of either $v_i$ or $v_j$, then, output $v_k$ is not a $p$-collider. Else, we describe a modification of $\GGG$ to obtain the flow network $\tilde{\GGG}$. First, initialize $\tilde{\GGG}$ with $\GGG$. Remove all outgoing edges of $v_k$ from $\tilde{\GGG}$ and set the capacity of all incoming edges incident on $v_k$ to $1$. Add a node $T_{ij}$ along with the edges $T_{ij} \rightarrow v_i$ and $T_{ij} \rightarrow v_j$ to $\tilde{\GGG}$ and set the capacity of these edges to $1$. For every node $w \in V \cup L \setminus \{ v_k \}$, create two nodes $w_{in}$ and $w_{out}$. Add edge $w_{out} \rightarrow w_{in}$ with a capacity $1$. Every incoming edge to $w$ i.e., $z \rightarrow w$ is replaced by $z \rightarrow w_{in}$ and every outgoing edge $w \rightarrow z$ is replaced by $w_{out} \rightarrow z$ with capacity $1$. Find maximum $s,t$ flow in $\tilde{\GGG}$ with $T_{ij}, v_k$ as source and sink respectively. If the maximum flow is $2$, then output $v_k$ is a $p$-collider, otherwise no.

\noindent Now, we outline the idea for the proof of correctness of the above construction.

\vspace{2mm}
\noindent\textsc{Sketch of the Proof}. After ensuring that $v_k$ has a directed path to either $v_i$ or $v_j$, we want to check whether there is an undirected path from $v_i$ to $v_j$ containing $v_k$ as a collider. In other words, we want to check if there are \textit{two vertex disjoint paths} from $v_i$ and $v_j$ to $v_k$ such that both of these paths have incoming edges to $v_k$. By adding a node $T_{ij}$ connected to $v_i$ and $v_j$, we want to route two units of flow from $T_{ij}$ to $v_k$ where each node has a vertex capacity of $1$. Converting vertex capacities into edge capacities by splitting every node into two nodes (one for incoming and the other for outgoing edges) gives us the desired flow network on which we can solve maximum flow.

\begin{figure}[h]
\begin{center}
\begin{tabular}{ccc}
\begin{subfigure}{0.33\textwidth}\centering\includegraphics[width=2in]{./b5.pdf}
\label{fig:1pc}
\end{subfigure}

\begin{subfigure}{0.33\textwidth}\centering\includegraphics[width=2in]{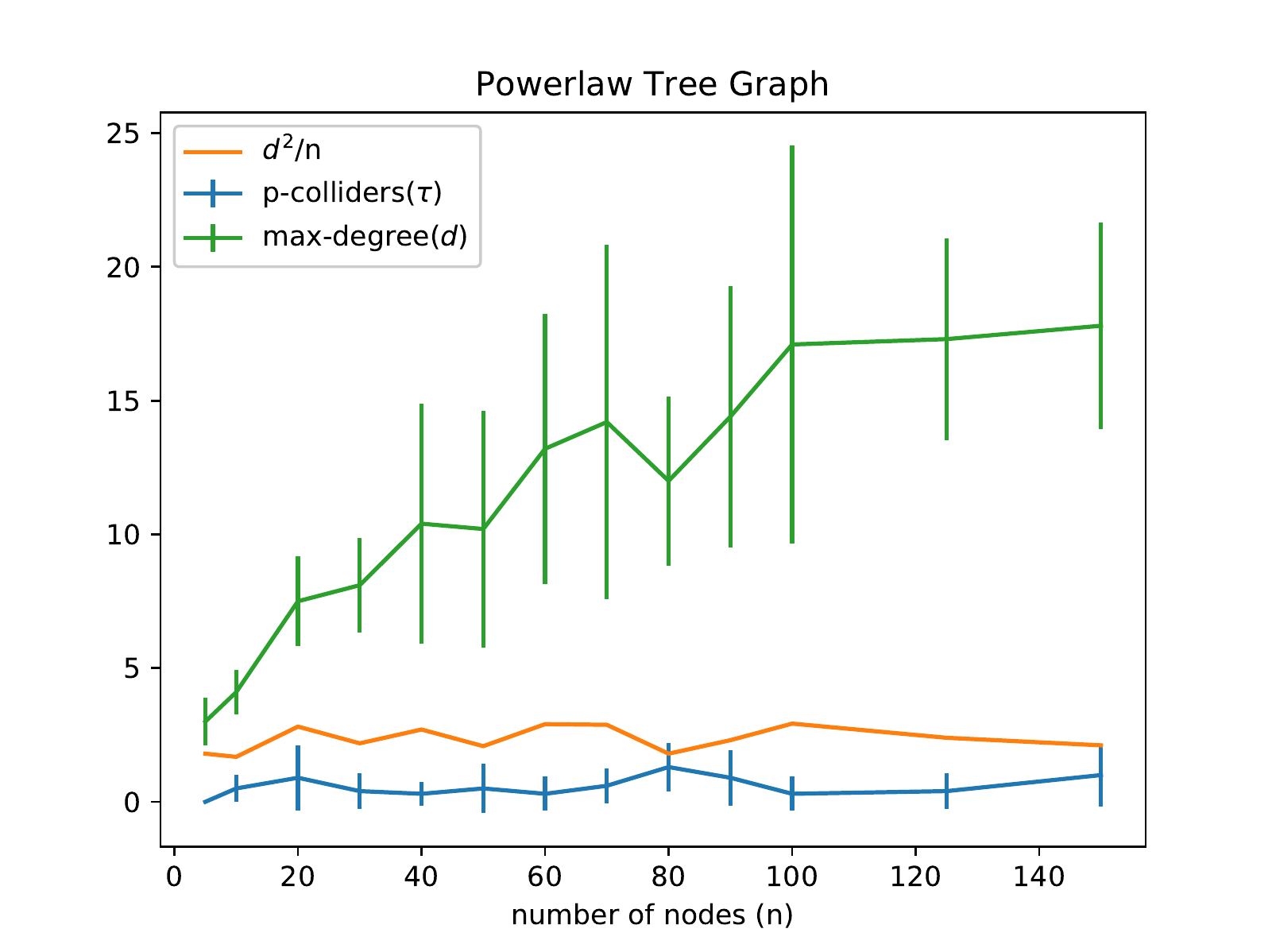}
\label{fig:1pd}
\end{subfigure}

\begin{subfigure}{0.33\textwidth}\centering\includegraphics[width=2in]{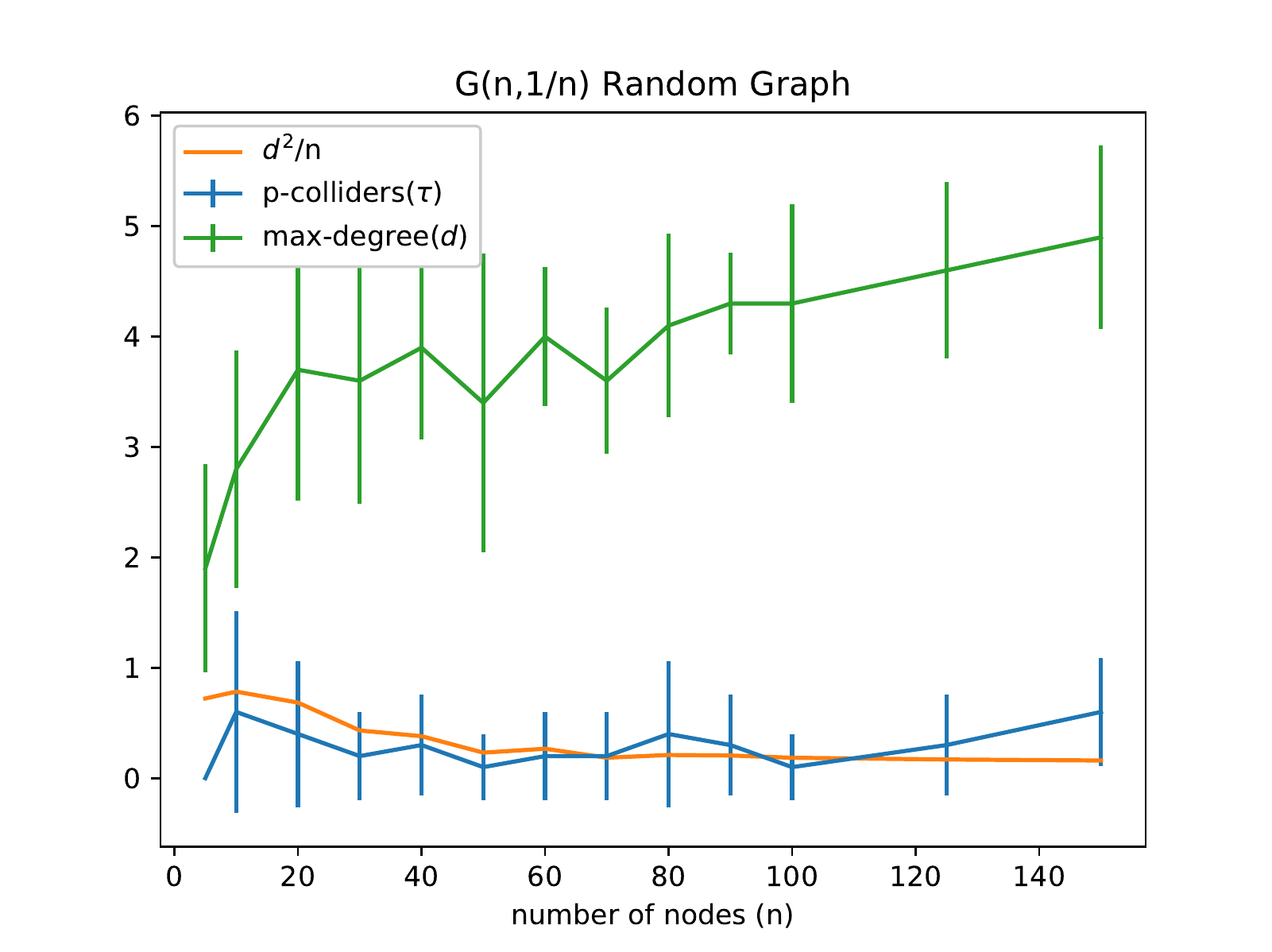}
\label{fig:1pb}
\end{subfigure}
\end{tabular}
\caption{Comparison of $\tau$ vs.\ maximum degree in various sparse random graph models. On the x-axis is the number of nodes in the graph. Note that our bound on the number of interventions needed to recover $\mathcal{G}$ is better than those provided by~\citep{neurips17} roughly when $\tau < d^2/n$.}
\label{fig:plot}
\end{center}
\end{figure}
\vspace{2mm}
\noindent\textbf{Results.}
In our plots (Figure~\ref{fig:plot}), we compare the maximum undirected degree $(d)$ with the maximum number of $p$-colliders between any pair of nodes (which defines $\tau$).
We ran each experiment $10$ times and plot the mean value along with one standard deviation error bars. 

 Recall that in the worst case, the number of interventions used by our approach (Theorem~\ref{thm:lat}) is $O(n\tau \log n + n \log n)$ while the algorithm proposed by~\citep{neurips17} uses $O(\min\{d \log^2 n ,\ell\} + d^2 \log n)$ many interventions where $\ell$ is the length of the longest directed path in the graph. So roughly when $\tau < d^2/n$, our bound is better. For this purpose, we also plot the $d^2/n$ line using the mean value of $d$ obtained.

For random bipartite graphs, that can be used to model causal relations over time, we use equal partition sizes $n_1 = n_2 = n/2$ and plot the results for $\GGG(n/2,n/2, c/n)$ for constant $c = 5$. We observe that the behaviour is uniform for small constant values of $c$. In this case, we observe that the number of $p$-colliders is close to zero for all values of $n$ in the range considered and our bound is better.

For directed random trees where degrees of nodes follow the powerlaw distribution (observed in real world networks~\citep{adamic2000power}), we again observe that for almost all the values of $n$, our bound is better. We run our experiments with small constant values for the exponent $\gamma$ and show the plots for $\gamma = 3$ in Figure~\ref{fig:plot}.

Powerlaw graphs contain only a few nodes concentrated around a very high degree. Therefore, we expect our algorithm to perform better in such cases. 

Also for Erd\"os-R\'enyi random graphs $\GGG(n,1/n)$, we observe that our bound is either better or comparable to that of~\citep{neurips17}.

It is interesting to see that in the sparse graphs we considered $\tau$ is considerably smaller compared to $d$.  Moreover, if we want to identify only the observable graph $G$ under the presence of latents, our algorithm uses $O(\tau \log n)$ interventions where as the previous known algorithm~\citep{neurips17} uses $O(d \log^2 n)$ interventions. In the random graphs considered above, our algorithms perform significantly better for identifying $G$. Therefore, we believe that minimizing the number of interventions based on the notion of $p$-colliders is a reasonable direction to consider.

\end{document}